\documentclass[11pt]{article}
\usepackage{amsmath,amssymb,amsfonts,amsthm,bbm}
\usepackage{verbatim}
\usepackage{setspace}
\usepackage{graphicx}
\usepackage{subcaption}
\usepackage{comment}
\usepackage{enumitem}
\usepackage[pagebackref,letterpaper=true,colorlinks=true,pdfpagemode=none,citecolor=OliveGreen,linkcolor=BrickRed,urlcolor=BrickRed,pdfstartview=FitH,linktocpage=true]{hyperref}
\usepackage[dvipsnames]{xcolor}
\usepackage{bbm}
\usepackage{algpseudocode}
\usepackage{algorithm} 
\usepackage[capitalise]{cleveref}

\newif\ifStat
\newif\ifNote
\newif\ifComments
\newif\ifThmitalic

\Stattrue
\Notefalse
\Commentstrue
\Thmitalictrue

\ifComments
\newcommand{\Tnote}[1]{{\color{blue} [Tselil: #1]}}
\else
\newcommand{\Tnote}[1]{}
\fi

\ifStat

\renewcommand{\tilde}{\widetilde}
\renewcommand{\bar}{\overline}
\renewcommand{\top}{{\sf T}}
\else

\fi
\newcommand{\iidsim}{\overset{\scriptsize\iid}{\sim}}

\usepackage[top=1in, bottom=1in, left=1in, right=1in]{geometry}

\makeatletter
\newcommand*{\rom}[1]{\expandafter\@slowromancap\romannumeral #1@}
\makeatother

\newcommand{\Holder}{H\"older}
\newcommand{\diag}{\mathrm{diag}}

\newcommand{\Tr}{\mathrm{Tr}}

\newcommand{\iid}{\mathrm{i.i.d.}}
\newcommand{\beq}{\begin{equation}}
\newcommand{\eeq}{\end{equation}}

\newcommand{\poly}{\mathrm{poly}}

\newcommand{\norm}[1]{\left\|{#1}\right\|}

\newcommand{\R}{\mathbb{R}}

\newcommand{\E}{\mathbb{E}}

\newcommand{\veps}{\varepsilon}
\newcommand{\tx}{\tilde{x}}

\newcommand{\tW}{\tilde{W}}

\renewcommand{\P}{\mathbb{P}}
\renewcommand{\S}{\mathbb{S}}
\newcommand{\NN}{\mathbb{N}}

\newcommand{\bA}{\mathrm{\bf A}}

\newcommand{\bT}{\mathrm{\bf T}}

\newcommand{\cS}{\mathcal{S}}

\newcommand{\cF}{\mathcal{F}}

\newcommand{\tA}{\tilde{A}}
\renewcommand{\bA}{\bar{A}}
\newcommand{\txp}{\tx^{\perp}}
\newcommand{\fp}{f^{\perp}}
\newcommand{\xp}{x^{\perp}}

\def\normal{{\sf N}}

\newcommand{\sN}{\mathsf{N}}

\ifNote

\theoremstyle{plain}
\newtheorem{thm}{Theorem}
\newtheorem{claim}{Claim}
\newtheorem{defn}{Definition}
\newtheorem{lem}[thm]{Lemma}
\newtheorem{cor}[thm]{Corollary}
\newtheorem{prop}[thm]{Proposition}
\newtheorem{ass}{Assumption}
\newtheorem{fact}{Fact}
\theoremstyle{definition}
\newtheorem{exm}{Example}
\newtheorem{rem}{Remark}
\newtheorem{conj}{Conjecture}

\else\ifThmitalic

\newtheorem{thm}{Theorem}[section]

\newtheorem{lem}[thm]{Lemma}

\newtheorem{prop}[thm]{Proposition}

\theoremstyle{definition}

\newtheorem{rem}{Remark}[section]

\else

\fi
\fi

\usepackage{mathrsfs}
\usepackage{setspace} 
\allowdisplaybreaks
\usepackage{todonotes}

\begin{document}
	
\title{Lower Bounds for the Convergence of Tensor Power Iteration on Random Overcomplete Models}
\author{
	Yuchen Wu\thanks{Department of Statistics, Stanford University. \href{mailto:wuyc14@stanford.edu}{\tt wuyc14@stanford.edu}.}
	\and 
	Kangjie Zhou\thanks{Department of Statistics, Stanford University. \href{mailto:kangjie@stanford.edu}{\tt kangjie@stanford.edu}.}
}
\date{\today}
\pagenumbering{arabic}
\maketitle

\begin{abstract}



Tensor decomposition serves as a powerful primitive in statistics and machine learning, and has numerous applications in problems such as learning latent variable models or mixture of Gaussians. In this paper, we focus on using power iteration to decompose an overcomplete random tensor.
Past work studying the properties of tensor power iteration either requires a non-trivial data-independent initialization, or is restricted to the undercomplete regime. Moreover, several papers implicitly suggest that logarithmically many iterations (in terms of the input dimension) are sufficient for the power method to recover one of the tensor components.

Here we present a novel analysis of the dynamics of tensor power iteration from random initialization in the overcomplete regime, where the tensor rank is much greater than its dimension. Surprisingly, we show that polynomially many steps are necessary for convergence of tensor power iteration to any of the true component, which refutes the previous conjecture. On the other hand, our numerical experiments suggest that tensor power iteration successfully recovers tensor components for a broad range of parameters in polynomial time. To further complement our empirical evidence, we prove that a popular objective function for tensor decomposition is strictly increasing along the power iteration path. 

Our proof is based on the Gaussian conditioning technique, which has been applied to analyze the approximate message passing (AMP) algorithm. The major ingredient of our argument is a conditioning lemma that allows us to generalize AMP-type analysis to non-proportional limit and polynomially many iterations of the power method.

\end{abstract}

\newpage
\tableofcontents

\newpage
\section{Introduction}

Tensors of order $m$ are multidimensional arrays with $m$ indices, with $m = 1$ corresponding to vectors and $m = 2$ corresponding to matrices. The notion of \emph{rank} naturally generalizes from matrices to tensors: An $m$-th order tensor $\bT \in (\R^d)^{\otimes m}$ is said to be rank-1 if it can be written as
\begin{equation*}
    \bT = v_1 \otimes \cdots \otimes v_m \Longleftrightarrow \bT (i_1, \cdots, i_m) = v_1 (i_1) \cdots v_m (i_m),
\end{equation*}
where $v_1, \cdots, v_m \in \R^d$. Past results imply that any tensor can be expressed as the sum of rank-1 tensors \cite{kiers2000towards,carroll1970analysis}. Namely, given $\bT \in (\R^d)^{\otimes m}$ we can find vectors $\{ v_{i}^{(j)} \}_{i \in [m], \ j \in [k]}$ such that
\begin{equation*}
    \bT = \sum_{j=1}^{k} v_1^{(j)} \otimes \cdots \otimes v_m^{(j)}.
\end{equation*}
The above decomposition is referred to as  \emph{tensor decomposition}, and the (CP) rank of a tensor is defined as the minimum number of rank-1 tensors required in such  decomposition. Unlike matrix decomposition, tensor decomposition with $m \geq 3$ is in many cases unique \cite{kruskal1977three}.  This is often true even in the overcomplete case, where the rank of the tensor is much larger than its ambient dimension. The uniqueness of tensor decomposition makes its application suitable in many practical settings, which we discuss below.

Tensor decomposition serves as a powerful primitive in statistics and machine learning, especially for algorithms that leverage \emph{the method of moments} \cite{pearson1894contributions} to learn model parameters. 
Applications of tensor decomposition include dictionary learning \cite{barak2015dictionary,ma2016polynomial,schramm2017fast}, Gaussian mixture models \cite{anandkumar2014tensor,ge2015learning,hsu2013learning}, independent component analysis \cite{de2007fourth,comon2010handbook}, and learning two-layer neural networks \cite{novikov2015tensorizing,mondelli2019connection}. Despite the fact that tensor decomposition is NP-hard in the worst case \cite{hillar2013most}, researchers have designed polynomial-time algorithms that successfully approximate the tensor components under natural distributional assumptions. Exemplary algorithms of this kind include the classical Jennrich's algorithm \cite{harshman1970foundations,de1996blind}, algebraic methods \cite{de2006link, de2007fourth}, iterative methods  \cite{zhang2001rank,anandkumar2014tensor,anandkumar2014guaranteed,anandkumar2015learning,anandkumar2017analyzing,kileel2019subspace,kileel2021landscape}, sum-of-squares (SOS) algorithms \cite{hopkins2015tensor,barak2015dictionary,ge2015decomposing,ma2016polynomial} and their spectral analogues \cite{hopkins2016fast,schramm2017fast,hopkins2019robust,ding2022fast}.

SOS algorithms and their spectral counterparts provably achieve strong guarantees of recovering tensor components, and can be implemented in polynomial time. However, they are often computationally prohibitive on large-scale problems due to the high-degree polynomial running time. Therefore, in practice it is often more preferable to resort to simple iterative algorithms \cite{celentano2020estimation,montanari2022statistically}, such as gradient descent and its variants. These algorithms are computationally efficient in terms of both runtime and memory, and are typically easy to implement. In the case of tensors, popular iterative algorithms include tensor power iteration \cite{anandkumar2017analyzing}, gradient descent on non-convex losses \cite{ge2017optimization}, and alternating minimization \cite{anandkumar2014guaranteed}. 

We focus in this paper on the tensor power iteration method, which can be regarded as a generalization of matrix power iteration. This method can also be viewed as gradient ascent on a polynomial objective function with infinite step size (see Eq.~\eqref{eq:poly_obj} for a formal definition). However, unlike the matrix case where the convergence properties are well understood theoretically, in the tensor case the dynamics of power iteration still remains mysterious due to non-convexity of the corresponding optimization problem. To unveil the mystery behind tensor power iteration, the present work proposes to study its asymptotic behavior on decomposing a random fourth order symmetric tensor
\begin{align*}
	\bT = \sum_{i = 1}^k a_i \otimes a_i  \otimes a_i \otimes a_i, \ a_i \in \R^d, \ \forall i \in [k]
\end{align*}
in the overcomplete regime $k \gg d$, where we assume $a_i  \iidsim \normal(0, I_d / d)$. 
We note that this is a well-studied model in the literature, while its properties are not yet fully understood. 
Given the entries of $\bT$, our goal is to recover one or all of the tensor components $\{a_i\}_{i \leq k}$, up to potential sign flips. 

We denote by $A \in \R^{k \times d}$ the matrix whose $i$-th row is $a_i^\top$. Initialized at $x_0 \in \S^{d - 1}(\sqrt{d})$ that is independent of the tensor components $\{a_i\}_{i \leq k}$, tensor power iteration is defined recursively as follows: 
\begin{equation}\label{eq:PI}
	x_t = \frac{\sqrt{d}\,\bT(I, x_{t - 1}, x_{t - 1}, x_{t - 1})}{\|\bT(I, x_{t - 1}, x_{t - 1}, x_{t - 1})\|_2}, \ t \ge 1,
\end{equation}
where
\begin{equation*}
	\bT(I, x, x, x) := \sum_{i,j,l \in [d]} x(i) x(j) x(l)  \bT(:, i, j, l).
\end{equation*}
For $x \in \S^{d-1} (\sqrt{d})$, we introduce the following polynomial objective function:
\begin{equation}\label{eq:poly_obj}
    \cS(x) = \sum_{i, j, k, l \in [d]} \bT(i, j, k, l) x(i) x(j) x(k) x(l) =  \sum_{i = 1}^k \langle a_i, x \rangle^4 = \norm{A x}_4^4.
\end{equation}
Notice that Eq.~\eqref{eq:PI} can be reformulated as $x_t = \sqrt{d} \cdot \nabla \cS(x_{t - 1}) / \norm{\nabla \cS(x_{t-1})}_2$, i.e., tensor power iteration can be regarded as  gradient ascent on $\cS$ with infinite learning rate. 

In the undercomplete regime where $k \leq d$, if the components $\{a_i\}_{i \leq k}$ are orthogonal to each other, then $\cS$ has only $2 k$ local maximizers that are close to $\{\pm \sqrt{d} a_i \}_{i \le k}$. In this case, tensor power iteration provably converges to one of the $\pm \sqrt{d} a_i$'s \cite{anandkumar2014tensor}. Indeed, tensors with linearly independent components can be orthogonalized, thus suggesting the existence of efficient algorithms in the undercomplete regime. 

Things become far more challenging in the overcomplete regime where $k$ is much greater than $d$. When $k \ll d^2$, it is known that any global maximizer of $\cS$ must be close to one of the $ \pm \sqrt{d} a_i$'s. However, algebraic geometry techniques show that $\cS$ has exponentially many other critical points \cite{cartwright2013number}. Further, if $k \gg d^2$, then $\cS (x)$ concentrates tightly around its expectation, uniformly for all $x \in \S^{d-1} (\sqrt{d})$, thus making it hard to identify the tensor components from the information encoded in $\cS$. As far as we know, there is no polynomial-time algorithm known for tensor decomposition within this regime. 

We thus focus on the regime $d \ll k \ll d^2$, where algebraic methods and SOS-based algorithms are proven to succeed \cite{ge2015decomposing,ma2016polynomial,hopkins2016fast,bhaskara2019smoothed,ding2022fast}, and numerical experiments indicate that the performance of randomly initialized power iteration matches that of these methods (see \cref{fig:success_prob} for details). 
However, from a theoretical standpoint, the dynamics of tensor power iteration in the overcomplete regime still remain elusive. A reasonable first step towards solving this puzzle would be to understand how many iterations are necessary for power method to find one of the true components.

\subsection{Main results}
We hereby give a partial answer to the above question. In particular, we establish several new results on the behavior of tensor power iteration in the overcomplete regime. Our first theorem states that randomly initialized tensor power iteration requires at least polynomially many steps to converge to a true component:
\begin{thm}[Slow convergence from random start, informal, see \cref{thm:negative}]\label{thm:informal1}
    Assume that $k$, $d$ are large enough, and that $k \asymp d^c$ for some $c \in (3/2, 2)$. Then, there exists some $\eta > 0$ that only depends on $c$, such that with high probability the following happens: Tensor power iteration from random initialization fails to identify any true component of $\bT$ within $d^{\eta}$ steps.
\end{thm}
\noindent {\bf Related work.}
Let us pause here to make some comparisons between our result and prior work on the same model: The seminal paper \cite{anandkumar2015learning} shows that tensor power iteration with an SVD-based initialization converges in $O(\log d)$ steps for $k = O(d)$. In \cite{anandkumar2017analyzing}, the authors prove that tensor power method successfully recovers one of the $a_i$'s in $O(\log \log d)$ iterations, given that its initialization has non-trivial correlations with the true components.
As a comparison, our Theorem~\ref{thm:informal1} shows that the $O(\log d)$ bound on the number of iterations does not hold for randomly-initialized power iteration, and establishes that polynomially many steps are necessary for convergence in the overcomplete regime. To the best of our knowledge, this is the first result that provides a lower bound on the computational complexity of tensor power iteration. From a more fundamental point of view, we also show that tensor power iterates are ``trapped" in a small neighborhood around its initialization for polynomially many steps.

Although the power method fails to converge in logarithmic many steps as conjectured, we still believe that it will correctly learn one of the tensor components when $k \ll d^2$ within polynomial time. We present numerical evidence as \cref{fig:success_prob} in \cref{sec:numerical} to support our claim. Establishing a rigorous positive result for tensor power iteration in this regime is challenging, and we leave it as an interesting open question for future work. As an alternative, we present here a weaker result suggesting the correctness of tensor power iteration. 
\begin{thm}[Increasing objective function, informal, see \cref{thm:increasing}]\label{thm:informal2}
    For $d^{3/2} \ll k \ll d^2$, starting from random initialization, with high probability the objective function $\cS$ is strictly increasing along the power iteration path up to finitely many steps.
\end{thm}
According to \cite{ge2017optimization}, the tensor components are the only local maximizers of $\cS$ on a superlevel set that is slightly better than random initialization. Their result, together with \cref{thm:informal2} suggest that, in order to prove convergence of tensor power iteration, it suffices to show that the objective function $\cS$ eventually surpasses a small threshold determined in \cite{ge2017optimization}.
\vspace{1em}

\noindent{\bf Proof technique.} Our proof is based on the Gaussian conditioning technique, and is similar to the analysis of the Approximate Message Passing (AMP) algorithm \cite{bayati2011dynamics}. The majority of prior AMP theory can only accommodate a constant number of iterations (i.e., the number of iterations does not grow with the input size) and proportional asymptotics (in our case, this corresponds to assuming $k / d \to \delta \in (0, \infty)$). \cite{rush2018finite} moves beyond the constant regime and extends Gaussian conditioning analysis to $O(\log d / \log \log d)$ many steps. However, their results fall short of validity when targeting for polynomially many  iterations, which is essential in our context. More recently, \cite{li2022non} develops a non-asymptotic framework that enables the analysis of AMP up to $O(n / \mbox{polylog}(n))$ many iterations, while they focus exclusively on symmetric spiked model and require nontrivial initialization. 

The technical innovation in this paper is that we successfully apply the Gaussian conditioning scheme to analyze the tensor power iteration up to polynomially many steps. Indeed, our conditioning lemma (Lemma~\ref{lemma:conditioning}) gives an exact non-asymptotic characterization of the power iterates, thus allowing for precisely tracking the values of the objective function along the iteration path. It is noteworthy that the same argument can be used to prove that polynomially many power iterates are necessary for general even-order tensors (see Remark~\ref{rem:negative}). 
To the best of our knowledge, this is the first result that generalizes AMP-type analysis to non-proportional asymptotics and polynomially many iterations. From a technical perspective, we believe our results will help to push forward the development of AMP theory and enrich the toolbox for theoretical analysis of general iterative algorithms.

\vspace{1em}
\noindent{\bf Organization.} The rest of this paper is organized as follows. Section~\ref{sec:prelim} introduces some preliminaries regarding power iteration and Gaussian conditioning technique. In Section~\ref{sec:overview} we state formally our main theorems and sketch their proofs, with all technical details deferred to the appendices. Section~\ref{sec:numerical} provides some useful numerical experiments that support our theoretical results. We provide in Section~\ref{sec:conclusion} several concluding remarks and discuss possible future directions.

\section{Preliminaries}\label{sec:prelim}

\subsection{Notation}

For $x \in \R^d$ and $S \in \R^{d \times k}$, we denote by $\Pi_{S}(x)$ the projection of $x$ onto the column space of $S$, and let $\Pi_S^{\perp}(x) = x - \Pi_S(x)$. For two vectors $u$ and $v$ of the same dimension, we use $\langle u, v \rangle$ to represent their inner product, and denote by $\norm{u}_p$ the $\ell^p$-norm of $u$ for $p \ge 1$. We use the notation $\odot$ to represent the Hadamard product of vectors and matrices. Furthermore, for $x \in \R^d$ we define $x^k = {x \odot x \cdots \odot x}$, the Hadamard product of $k$ copies of $x$. For a matrix $X$, we denote by $X^{\dagger}$ the pseudoinverse of $X$. We denote by $e_k$ the $k$-th standard basis in any Euclidean space.

We denote by $\mathbb{S}^{d - 1}(r)$ the sphere in $\R^d$ centered at the origin with radius $r$. For random variables $X, Y$, we write $X \perp Y$ if $X$ is independent of $Y$. For a collection of random variables $\{ X_i \}_{i \in I}$, we use $\sigma (\{ X_i \}_{i \in I})$ to represent the $\sigma$-algebra generated by these random variables.  

For $n \in \NN_+$, we define the set $[n] := \{1,2,\cdots, n\}$. For two sequences of positive numbers $\{a_n\}_{n \in \NN_+}$, $\{b_n\}_{n \in \NN_+}$, we say $a_n \ll b_n$ if $a_n / b_n \to 0$ as $n \to \infty$, and $a_n \asymp b_n$ if $a_n = O(b_n)$ and $b_n = O(a_n)$. For $\{c_n\}_{n \in \NN_+} \subseteq \R$, we say $c_n = o_n(1)$ if $c_n \to 0$ as $n \to \infty$. For a sequence of events $\{E_n\}_{n \in \NN_+}$, we say that $E_n$ happens with high probability if $\P(E_n) = 1 - o_n(1)$.

We reserve $O_P$ and $o_P$ as the standard big-$O$/small-$o$ in probability notations: For a set of random variables $\{X_n\}_{n \geq 1}$ and a sequence of positive numbers $\{a_n\}_{n \geq 1}$, we say $X_n = o_P(a_n)$ if and only if for all $\delta > 0$, $\lim_{n \to \infty} \P(|X_n / a_n| > \delta) \to 0$. Similarly, we say $X_n = O_P(a_n)$ if and only if for all $\delta > 0$, there exists $M, N \in \R_+$, such that $\P(|X_n / a_n| > M) < \delta$ for all $n > N$.


Throughout the proof, with a slight abuse of notation, we use capital letter $C$ to represent various numerical constants, the values of which might not necessarily be the same in each occurrence.

\subsection{Tensor power iteration}

Using the rank-one decomposition of $\bT$, we can reformulate the tensor power iteration stated in \cref{eq:PI} as follows:
\begin{equation*}
	x_t = \sqrt{d} \cdot \frac{\sum_{i=1}^{k} \langle a_i, x_{t - 1} \rangle^3 a_i}{\norm{\sum_{i=1}^{k} \langle a_i, x_{t - 1} \rangle^3 a_i}_2}, \qquad t \geq 1.
\end{equation*}
For notational convenience, we recast $x_t$ as $\tx_{t}$, and redefine
\begin{equation*}
	x_t := \sum_{i=1}^{k} \langle a_i, \tilde{x}_{t - 1} \rangle^3 a_i = A^{\sf T} (A \tilde{x}_{t - 1})^3,
\end{equation*}
where we recall that $A \in \R^{k \times d}$ is the matrix whose $i$-th row is $a_i^\top$. In other words, $x_t$ is the gradient of $\cS (x) = \norm{A x}_4^4$ at $\tilde{x}_{t - 1}$, and $\tilde{x}_t$ is the projection of $x_t$ onto $\S^{d-1} (\sqrt{d})$:
\begin{equation}\label{eq:redef_TPI}
    x_t = A^\top (A \tilde{x}_{t-1})^3, \ \tilde{x}_t = \sqrt{d} \cdot \frac{x_t}{\norm{x_t}_2}.
\end{equation}
We introduce some useful intermediate variables: $y_t = A \tilde{x}_{t-1}$, $f_t = y_t^3$. Using these intermediate  variables, tensor power iteration can be decomposed into the following steps:
\begin{equation}\label{eq:compact_PI}
\begin{split}
	& y_t = A \tx_{t - 1}, \qquad f_t = y_t^3, \\
	& x_t = A^{\top} f_t, \qquad \tx_{t} = \sqrt{d} \cdot \frac{x_{t}}{\|x_{t}\|_2}. 
\end{split}
\end{equation}
%

\subsection{Intuition behind slow convergence}
We now provide a heuristic justification for \cref{thm:informal1} through analyzing the first step of power iteration. Indeed, we show that for any initialization $x_0 \in \S^{d-1} (\sqrt{d})$ that is independent of $\bT$, the normalized first iterate $\tilde{x}_1$ (defined in \cref{eq:redef_TPI}) must lie in a small neighborhood of $x_0$ with high probability. This claim is made precise by the following proposition:
\begin{prop}
	Let $x_0 \in \S^{d-1} (\sqrt{d})$ be independent of $A$, and $\tilde{x}_1$ be defined as per \cref{eq:redef_TPI}. Furthermore, we assume $k \asymp d^c$ for some $c \in (3/2, 2)$, and let $\eta$ be a small positive constant satisfying $\eta < (c-1) / 2$. Then, with probability at least $1 - C_0 \exp(- C_1 d^{\min\{ 2 \eta, c/4 + \eta / 2 \}} )$, it holds that $\norm{\tilde{x}_1 - x_0}_2 \le C_2 \cdot d^{(1-c)/2 + \eta} \cdot \norm{x_0}_2$, where $C_0, C_1, C_2 > 0$ are absolute constants.
\end{prop} 
\begin{proof}
	By rotational invariance, we may assume without loss of generality that $x_0 = \sqrt{d} \cdot e_1 = (\sqrt{d}, 0, \cdots, 0)^\top$, which implies that
	\begin{equation}
		\sum_{i=1}^{k} \langle a_i, x_0 \rangle^3 a_i = \sum_{i=1}^{k}d^{3/2} a_{i1}^3 a_i = \left(d^{3/2} \sum_{i=1}^{k} a_{i1}^4, d^{3/2} \sum_{i=1}^{k} a_{i1}^3 a_{i2}, \cdots, d^{3/2}\sum_{i=1}^{k} a_{i1}^3 a_{id} \right).
	\end{equation}
	Since $a_{ij} \sim_{\iid} \normal (0, 1/d)$, applying \cref{lemma:conc_non_exp} gives the following concentration bounds:
	\begin{align}
		& \P \left( \left\vert \sum_{i=1}^{k} a_{i 1}^4 - \frac{3 k}{d^2} \right\vert \ge t \right) \le C_0 \exp \left( - C_1 \cdot \min \left\{ \frac{t^2 d^4}{k}, \ d \sqrt{t} \right\} \right), \\
		& \P \left( \left\vert \sum_{i=1}^{k} a_{i1}^3 a_{i l} \right\vert \ge t \right) \le C_0 \exp \left( - C_1 \cdot \min \left\{ \frac{t^2 d^4}{k}, \ d \sqrt{t} \right\} \right), \ \text{for \ } l = 2, \cdots, d.
	\end{align}
	As a consequence, we know that
	\begin{align}
		& \P \left( \norm{\sum_{i=1}^{k} \frac{1}{d^{3/2}} \langle a_i, x_0 \rangle^3 a_i - \frac{3k}{d^2} e_1 }_2 \ge  t \sqrt{d} \right) \le C_0 d \exp \left( - C_1 \cdot \min \left\{ \frac{t^2 d^4}{k}, \ d \sqrt{t} \right\} \right)  \\
		\implies & \P \left( \norm{ \frac{d}{3 k} \cdot \sum_{i=1}^{k} \langle a_i, x_0 \rangle^3 a_i - x_0 }_2 \ge \frac{d^{3} t}{3 k} \right) \le C_0 d \exp \left( - C_1 \cdot \min \left\{ \frac{t^2 d^4}{k}, \ d \sqrt{t} \right\} \right).
	\end{align}
	Therefore, setting $t = 3ks / d^{2.5}$, we obtain that 
	\begin{align*}
		& \P \left( \left|\|x_1\|_2 - \frac{3k}{\sqrt{d}} \right| \geq \frac{3ks}{\sqrt{d}} \right) \leq C_0 d \exp \left( - C_1 \cdot \min \left\{ \frac{ks^2}{d}, \ \frac{\sqrt{ks}}{d^{1/4}} \right\} \right), \\
		& \P \left( \left\| x_1- \frac{3k}{\sqrt{d}} e_1 \right| \geq \frac{3ks}{\sqrt{d}} \right) \leq C_0 d \exp \left( - C_1 \cdot \min \left\{ \frac{ks^2}{d}, \ \frac{\sqrt{ks}}{d^{1/4}} \right\} \right).
	\end{align*}
	Hence, for any $s > 0$, 
	\begin{align*}
		\P\left( \|\tilde{x}_1 - x_0\|_2 \geq s \sqrt{d} \right) \leq \, & \P \left( \Big\|\tilde{x}_1 - \frac{d}{3k} \cdot x_1 \Big\|_2 \geq \frac{s \sqrt{d}}{2}  \right) + \P\left( \Big\| \frac{d}{3k} \cdot x_1 - x_0 \Big\|_2 \geq \frac{s \sqrt{d}}{2}\right) \\
		= \,& \P\left( \left| \|x_1\|_2 - \frac{3k}{\sqrt{d}} \right| \geq \frac{3ks}{2\sqrt{d}}  \right) + \P\left( \Big\| \frac{d}{3k} \cdot x_1 - x_0 \Big\|_2 \geq \frac{s \sqrt{d}}{2}\right) \\
		\leq \, & C_0 d \exp \left( - C_1 \cdot \min \left\{ \frac{ks^2}{d}, \ \frac{\sqrt{ks}}{d^{1/4}} \right\} \right),
	\end{align*}
	which further implies that for any $s > 0$,
	\begin{equation}
		\P \left( \norm{\tilde{x}_1 - x_0}_2 \ge s\sqrt{d} \right) \le C_0 d \exp \left( - C_1 \cdot \min \left\{ \frac{k s^2}{d}, \frac{\sqrt{k s}}{d^{1/4}} \right\} \right).
	\end{equation}
	Recall that $k \asymp d^c$, then choosing $s = d^{(1- c) / 2 + \eta}$ yields the desired result. 
\end{proof}

The above proposition implies that $\norm{\tilde{x}_1 - x_0}_2 / \norm{x_0}_2$ is polynomially small in $d$ with high probability ($1 - \exp(- \Omega (d^\veps))$. If we can establish similar upper bounds for power iterations up to $t = \poly(d)$, then we are in good shape as it requires at least polynomially many iterates for the power method to escape the neighborhood of $x_0$ of (an arbitrarily small) constant radius and in turn converges to any of the tensor components. In the following sections, we show that this is indeed the case by leveraging the Gaussian conditioning argument and extend its analysis to polynomially many steps.

\subsection{Gaussian conditioning}
In this section, we present a Gaussian conditioning lemma, which enables us to decompose each step of the power iteration as the sum of projections onto its previous iterates and an independent component. This lemma can be viewed as a multi-step generalization of Lemma~3.1 in \cite{montanari2022adversarial}, and is proved using the properties of Gaussian conditional distribution.

Recalling the variables defined in Eq.~\eqref{eq:compact_PI}, we further denote by $F_{1:t} \in \R^{k \times t}$ the matrix whose $i$-th column is $f_i$, and $X_{0:t} \in \R^{d \times (t + 1)}$ the matrix whose $j$-th column is $x_{j - 1}$. We define $f_t^{\perp} = \Pi_{F_{1:{t - 1}}}^{\perp}(f_t)$, $x_t^{\perp} = \Pi_{X_{0:t- 1}}^{\perp} (x_t)$, and $\tx_t^{\perp} = \Pi_{X_{0:t- 1}}^{\perp} (\tx_t)$. Note that $\fp_t \in \R^k$ and $\xp_t, \txp_t \in \R^d$.  The following lemma will be used several times throughout our proof:
\begin{lem}[Gaussian conditioning]\label{lemma:conditioning}
	For $s,t \in \NN$, we define the sigma-algebra:
	\begin{equation*}
	    \cF_{s,t} := \sigma ( x_0, \cdots, x_s, y_1, \cdots, y_t ).
	\end{equation*}
	Then, we have the following decompositions: 
\begin{align}
	 x_t = & \sum_{i = 0}^{t - 1} \txp_i \cdot \frac{\langle h_{i + 1}, f_t \rangle}{\|\txp_i\|_2^2} + \Pi_{X_{0:t - 1}}^{\perp} \bA_t^{\top} \fp_t, \label{eq:cond-xt}\\
	 \begin{split}\label{eq:cond-yt}
	 y_{t + 1} = & \sum_{i = 1}^t h_i \cdot \frac{\langle \tx_{i - 1}^{\perp}, \tx_t \rangle}{\|\tx_{i - 1}^{\perp}\|_2^2} + h_{t + 1} \\
	 = & \sum_{i = 1}^{t + 1} \Pi_{F_{1:i - 1}}^{\perp} \tilde{A}_i \tx_{i - 1}^{\perp} \frac{\langle \txp_{i - 1}, \tx_t \rangle}{\|\tx_{i - 1}^{\perp}\|_2^2} + \sum_{i = 2}^{t + 1} \fp_{i - 1} \cdot \frac{\langle x_{i - 1}, \txp_{i - 1} \rangle}{\|\fp_{i - 1}\|_2^2} \cdot \frac{\langle \txp_{i - 1}, \tx_t\rangle}{\|\txp_{i - 1}\|_2^2},
	 \end{split}
\end{align}
where $\bA_t, \tA_{t + 1} \in \R^{k \times d}$ satisfy $\bA_t \overset{d}{=} \tA_{t + 1} \overset{d}{=} A$, and
\begin{equation*}
    \bA_t \perp \sigma(\cF_{t - 1, t} \cup \sigma (\bA_1, \cdots, \bA_{t - 1}, \tA_1, \dots, \tA_t) ), \ \tA_{t + 1} \perp \sigma(\cF_{t,t} \cup \sigma (\bA_1, \cdots, \bA_{t}, \tA_1, \dots, \tA_{t}) ).
\end{equation*}
Further, we have
\begin{align}\label{eq:ht+1}
	h_{t + 1} = \fp_t \cdot \frac{\langle x_t, \txp_t \rangle}{\| \fp_t \|_2^2} + \Pi_{F_{1:t}}^{\perp} \tA_{t + 1} \txp_t.
\end{align}
\end{lem}
The proof of Lemma~\ref{lemma:conditioning} is provided in Appendix~\ref{sec:pf_cond_lem}.

\section{Proof overview}\label{sec:overview}
We give in this section a formal statement of our main theorem, and provide an overview of its proof. We postpone the full version of the proof to Appendix~\ref{sec:dynamics_TPI}.
\begin{thm}\label{thm:negative}
	Assume $d,k, T \to \infty$ simultaneously satisfying $d^{3/2} \ll k \ll d^2$, and that
	\begin{equation*}
		T = T(k, d) \ll (\log k)^{-1/3} \cdot \frac{k^{2/3}}{d}.
	\end{equation*} 
    Then for any $\veps > 0$, with probability $1 - o_d(1)$, the following result holds for all $0 \leq t \leq T$:
	\begin{align}\label{eq:main}
		\frac{1}{\sqrt{d}} \max_{i \in [k]}| \langle \tx_t, a_i \rangle | \leq \veps. 
	\end{align}
	That is to say, tensor power iteration fails to identify any true component in $T(k, d)$ steps.
\end{thm}
%
%
\begin{rem}\label{rem:negative}
    The above conclusion extends to the case of arbitrary even-order tensors: For general $m \ge 2$ and $a_i \iidsim \normal (0, I_d/d)$, let $\bT = \sum_{i=1}^{k} a_i^{\otimes 2m}$. Tensor power iteration for decomposing $\bT$ is defined as the following iteration:
    \begin{equation*}
        x_t = A^\top (A \tilde{x}_{t-1})^{2m - 1}, \ \tilde{x}_t = \sqrt{d} \cdot \frac{x_t}{\norm{x_t}_2}.
    \end{equation*}
    Similar to the proof of  Theorem~\ref{thm:negative}, we obtain that under the condition $d^{(2m-1)/m} \ll k \ll d^m$, as long as tensor power iteration starts from a random initialization and 
    \begin{equation*}
        T \ll (\log k)^{-1/3} \cdot \min \left\{ k^{1/2m}, \ \frac{k^{m/(2m-1)}}{d} \right\},
    \end{equation*}
    then for all fixed $\veps > 0$, with probability $1 - o_d(1)$, Eq.~\eqref{eq:main} holds for all $0 \le t \le T$. The proof of this claim is similar to the proof of \cref{thm:negative}, and we skip it here for the sake of simplicity. 
\end{rem}

\begin{rem}
	Using a similar argument, we can show that projected gradient descent requires at least polynomially many steps to converge to any tensor component as well.
\end{rem}
%
%
We also state here a formal version of Theorem~\ref{thm:informal2}, whose proof is provided in Appendix~\ref{sec:pf_increasing}:
\begin{thm}\label{thm:increasing}
	Recall that $\cS(x) = \sum_{i = 1}^k \langle a_i, x\rangle^4$ for $x \in \S^{d-1} (\sqrt{d})$. Assume $d,k \to \infty$ simultaneously satisfying $d^{3/2} \ll k \ll d^2$. Then for any $T_c \in \NN_+$ that does not grow with $k$ and $d$, the following holds with high probability: For all $t = 0,1,\cdots, T_c$, we have
	\begin{align*}
		S(\tx_t) = 3k + 20td + o_P(d).
	\end{align*}
	As a consequence, with probability $1 - o_d(1)$ we have $\cS(\tx_{t + 1}) > \cS(\tx_t)$ for all $0 \le t \le T_c - 1$.
\end{thm}

\subsection{Proof sketch of Theorem~\ref{thm:negative}}
By definition, we have $y_t = A \tilde{x}_{t-1}$. Therefore, in order to prove Theorem~\ref{thm:negative}, it suffices to show with probability $1 - o_d(1)$, $\norm{y_{t}}_{\infty} \le \veps \sqrt{d}$ for all $0 \leq t \leq T$. We will prove a stronger version of this result, namely $\norm{y_t}_4 \le \veps \sqrt{d}$. To this end, we use Eq.~\eqref{eq:cond-yt} in Lemma~\ref{lemma:conditioning} and decompose $y_t$ as $y_t = w_t - \eta_t + v_t$, where
\begin{align*}
    & w_t = \sum_{i = 1}^t \tA_i \txp_{i - 1} \cdot \frac{\langle \txp_{i - 1}, \tx_{t - 1} \rangle}{\|\txp_{i - 1}\|_2^2}, \\
	& \eta_t = \sum_{i = 1}^t \sum_{j = 1}^{i - 1} \fp_j \cdot \frac{\langle \fp_j, \tA_i \txp_{i - 1} \rangle}{\|\fp_j\|_2^2} \cdot \frac{\langle \txp_{i - 1}, \tx_{t - 1} \rangle}{\|\txp_{i - 1}\|_2^2}, \\
	& v_t = \sum_{i = 2}^t \fp_{i - 1} \cdot \frac{\langle x_{i - 1}, \txp_{i - 1}\rangle}{\|\fp_{i - 1}\|_2^2} \cdot \frac{\langle \txp_{i - 1}, \tx_{t - 1} \rangle}{\|\txp_{i - 1}\|_2^2}.
\end{align*}
By triangle inequality, we have
\begin{equation*}
    \norm{y_t}_4 \le \norm{w_t}_4 + \norm{\eta_t}_4 + \norm{v_t}_4 \le \norm{w_t}_4 + \norm{\eta_t}_2 + \norm{v_t}_2.
\end{equation*}
Then, it suffices to upper bound $\norm{w_t}_4$, $\norm{\eta_t}_2$, and $\norm{v_t}_2$ respectively.

\vspace{1em}
\noindent {\bf Upper bounding $\norm{w_t}_4$ and $\norm{\eta_t}_2$.} For future convenience, define
\begin{equation*}
		z_i = \tA_i \txp_{i - 1} \cdot \frac{\sqrt{d}}{\|\txp_{i - 1}\|_2}, \qquad \alpha_{i, t} = \frac{\langle \txp_{i - 1}, \tx_{t - 1} \rangle}{\sqrt{d} \|\txp_{i - 1}\|_2}. 
\end{equation*}
Then, we can write $w_t = \sum_{i=1}^{t} \alpha_{i, t} z_i$ and $\eta_t = \sum_{i=2}^{t} \alpha_{i, t} \Pi_{F_{1:i-1}} z_i$. Moreover, one can show that $z_i \iidsim \normal (0, I_k)$ and $\sum_{i=1}^{t} \alpha_{i, t}^2 = 1$. Using standard probability tools, we obtain that there exists a numerical constant $C>0$, such that the following result holds with high probability (see Lemma~\ref{lemma:wt} and~\ref{lemma:etat} for more details):
\begin{equation*}
    \norm{w_t}_4 \le C k^{1/4} \ll \sqrt{d}, \qquad \norm{\eta_t}_2 \le C T^{3/4} (\log k)^{1/4} \ll \sqrt{d} \qquad \forall t \in [T].
\end{equation*}

\vspace{1em}
\noindent {\bf Upper bounding $\norm{v_t}_2$.} First, using Lemma~\ref{lemma:norm-of-vt} in the appendix, we can show that $\norm{v_{t+1}}_2^2$ is very close to $\norm{\Pi_{x_0}^\perp (\tilde{x}_t)}_2^2$. Therefore, it suffices to prove that
\begin{equation*}
    \norm{\Pi_{x_0}^\perp (\tilde{x}_t)}_2^2 \ll d = \norm{\tilde{x}_t}_2^2, \qquad \forall t \in [T],
\end{equation*}
i.e., tensor power iteration stays close to its initialization in the first $T$ steps. To this end, we express $\norm{\Pi_{x_0}^\perp (\tilde{x}_t)}_2^2$ as $d \cdot \norm{\Pi_{x_0}^\perp (x_t)}_2^2 / \norm{x_t}_2^2$, and estimate $\norm{\Pi_{x_0}^\perp (x_t)}_2^2$ and $\norm{x_t}_2^2$, respectively. Leveraging Lemma~\ref{lemma:conditioning}, we can express $x_t$ in terms of the previous iterations, and finally establish the following recurrence relation (see Lemma~\ref{lemma:term-II}-\ref{lemma:term-I} for details):
\begin{equation*}
    \frac{\norm{\Pi_{x_0}^\perp (x_t)}_2^2}{\norm{x_t}_2^2} \le \frac{\norm{\Pi_{x_0}^\perp (x_t)}_2^2}{\norm{\Pi_{x_0} (x_t)}_2^2} \le U_{k, d, T} \left( \frac{\norm{\Pi_{x_0}^\perp (x_{t-1})}_2^2}{\norm{\Pi_{x_0} (x_{t-1})}_2^2} \right),
\end{equation*}
where $U_{k, d, T}$ is an increasing function. See the proof of Theorem~\ref{thm:negative} for its explicit definition.

\vspace{1em}
\noindent {\bf Analysis of $U_{k, d, T}$.} Note that the above inequality and the properties of $U_{k, d, T}$ imply:
\begin{equation*}
    \frac{\norm{\Pi_{x_0}^\perp (x_t)}_2^2}{\norm{\Pi_{x_0} (x_t)}_2^2} \le U_{k, d, T}^t \left( \frac{\norm{\Pi_{x_0}^\perp (x_0)}_2^2}{\norm{\Pi_{x_0} (x_0)}_2^2} \right) = U_{k, d, T}^t \left( 0 \right), \quad \text{for all $t \in [T]$}.
\end{equation*}
Hence, it suffices to show that $U_{k, d, T}^t \left( 0 \right) \ll 1$ for all $t \in [T]$. This is achieved by establishing that
\begin{equation}\label{eq:up_bd_fct_U}
    U_{k, d, T} (x) \le \left( 1 + \frac{1}{2 T} \right) x, \qquad \forall x \in [\veps_{k, d}, 2 \veps_{k, d}],
\end{equation}
where $\veps_{k, d} \ll 1$ explicitly depends on $k$ and $d$. It finally follows that
\begin{equation*}
    U_{k, d, T}^t \left( 0 \right) \le U_{k, d, T}^t \left( \veps_{k, d} \right) \le \left( 1 + \frac{1}{2T} \right)^t \cdot \veps_{k, d} \le \sqrt{e} \veps_{k, d} \le 2 \veps_{k, d},
\end{equation*}
where the last inequality guarantees that Eq.~\eqref{eq:up_bd_fct_U} can be repeatedly applied to control $U_{k, d, T}^t \left( \veps_{k, d} \right)$. This completes the proof sketch of Theorem~\ref{thm:negative}.

\section{Numerical experiments}\label{sec:numerical}

We present in this section several numerical experiments that support our theory. In the following experiments, we generate random tensors, and run tensor power method \eqref{eq:PI} starting from random initialization over the unit sphere.
If the iterates converge to one of the components (up to signs), then we call it a success. 

We investigate the success region for tensor power method in our first experiment. In this experiment, we run power method for 1000 iterations, and repeat this procedure for 1000 times independently for each pair of $(k, d)$. We record the corresponding empirical success rates, and plot these success rates for different $(k, d)$ as a heat map shown by \cref{fig:success_prob}. From the plot, we see that the success rate undergoes a sharp phase transition around the boundary $\log k = 2 \log d$. Our experiment suggests  that for $k \ll d^2$, tensor power iteration succeeds with high probability, and when $k \gg d^2$ it fails with high probability. This matches the success region of SOS-based methods, which is also the conjectured region for possible recovery with polynomial-time algorithms.
\begin{figure}[h!]
	\centering
	\includegraphics[width = 0.6\linewidth]{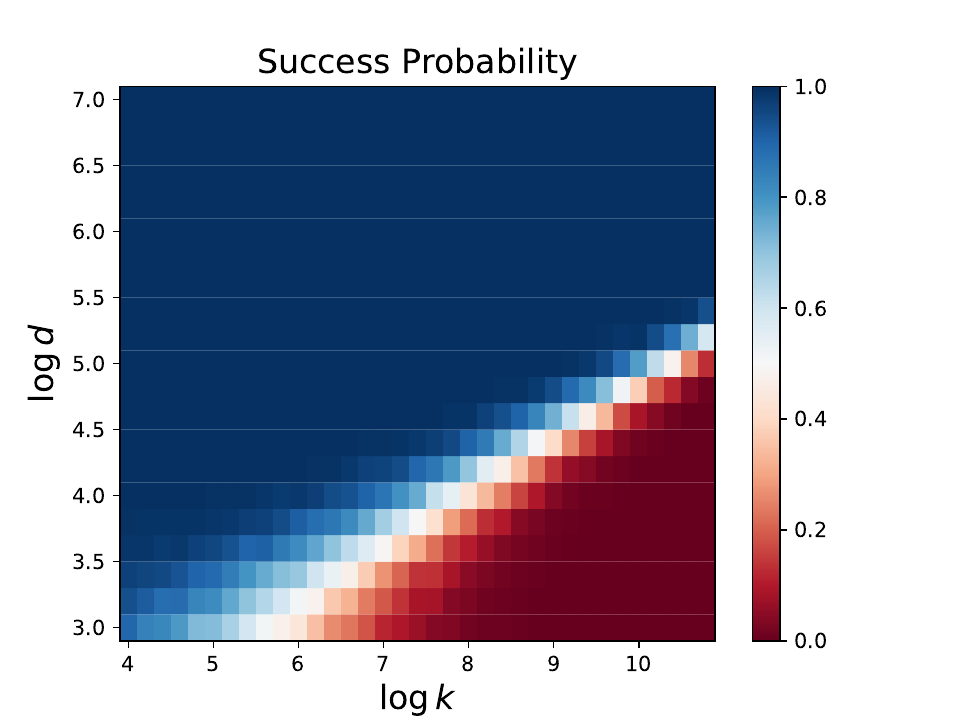}
	\caption{Success probability of tensor power iteration for varying $k$ and $d$. }
	\label{fig:success_prob}
\end{figure}

In our second experiment, we fix $\log k / \log d = 1.8$, and varies $d$ and $T$. Again we repeat the experiment 1000 times for each pair of $(d, T)$, and record the estimated success probability after $T$ steps of power iteration. We present the outcomes in \cref{fig:success_prob2}. The heatmap shows that polynomially many steps are necessary for tensor power iteration to converge, since $\log T$ scales linearly with $\log d$ around the success/failure boundary. This observation  validates  Theorem~\ref{thm:negative}.
\begin{figure}[h!]
	\centering
	\includegraphics[width = 0.6\linewidth]{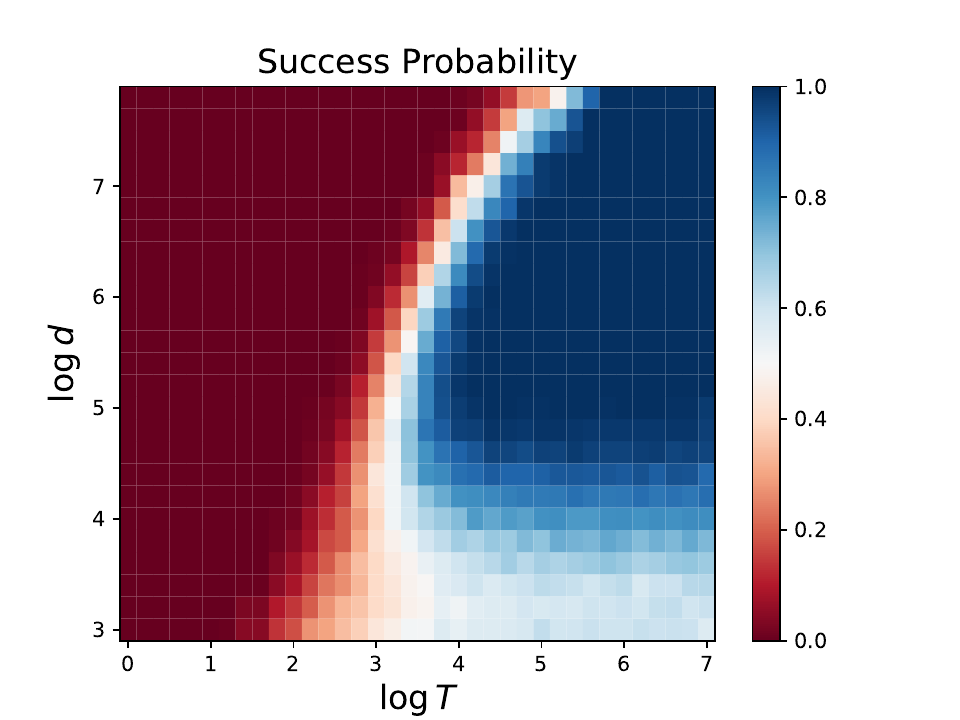}
	\caption{Success probability of tensor power iteration for varying $d$ and $T$. }
	\label{fig:success_prob2}
\end{figure}

We illustrate \cref{thm:increasing} in our third experiment. In this experiment, we record the values of the score function $\mathcal{S}$ for the first few iterations along the power iteration path, and compare them with the corresponding theoretical predictions given by \cref{thm:increasing}. We repeat such procedure for 1000 times independently, and present the outcomes in \cref{fig:increasing}, which justifies the correctness of the theorem.

\begin{figure}[h!]
\minipage{0.35\textwidth}
  \includegraphics[width=\linewidth]{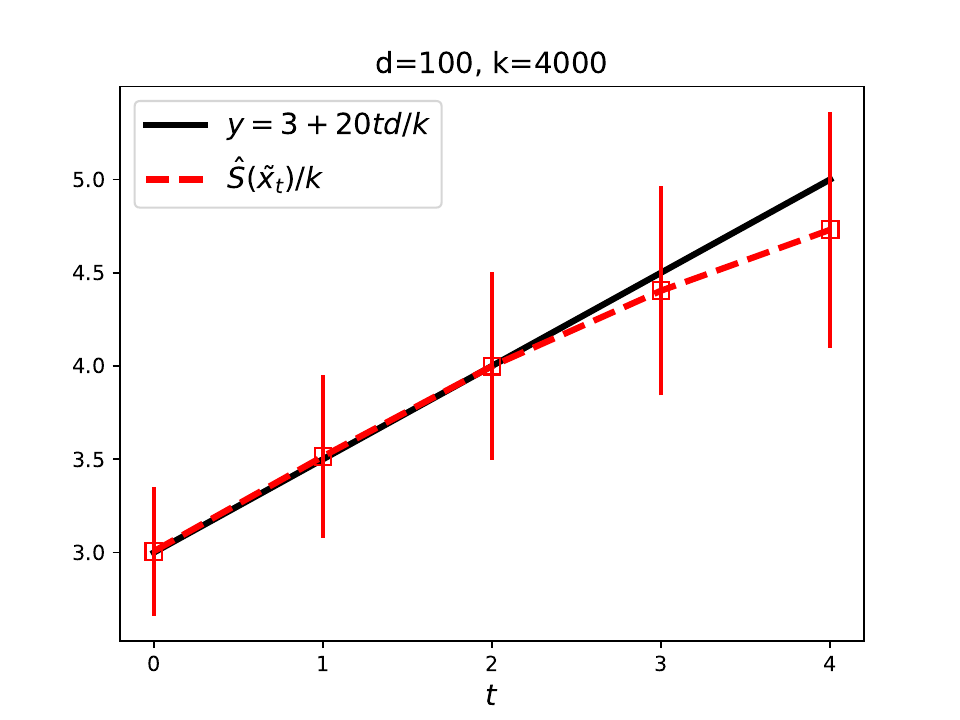}
\endminipage \hspace{-1cm} \hfill
\minipage{0.35\textwidth}
  \includegraphics[width=\linewidth]{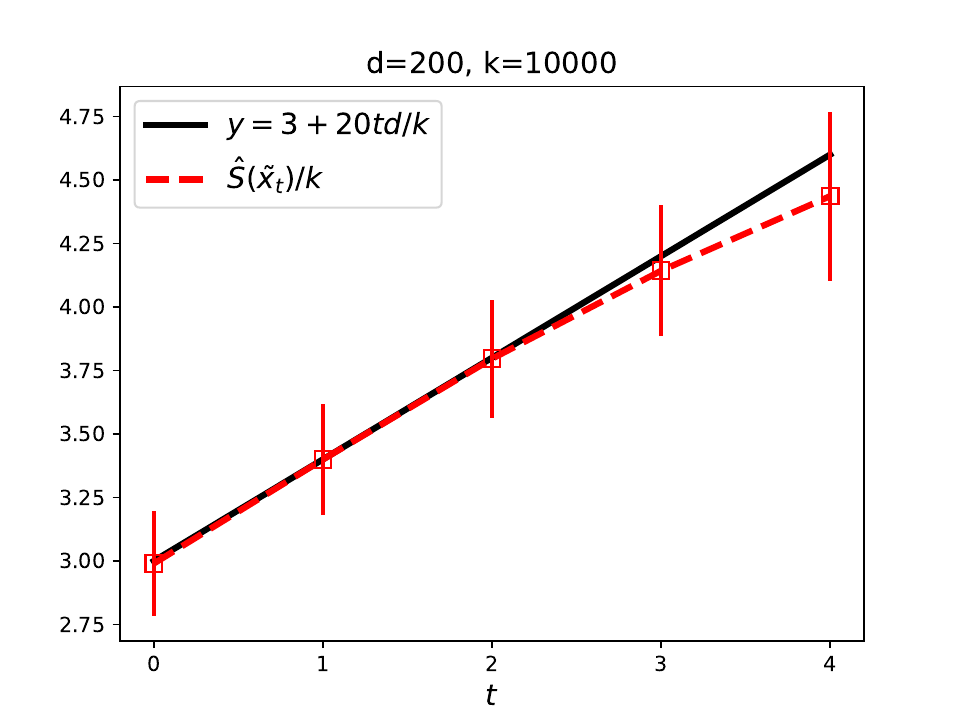}
\endminipage \hspace{-1cm} \hfill
\minipage{0.35\textwidth}%
  \includegraphics[width=\linewidth]{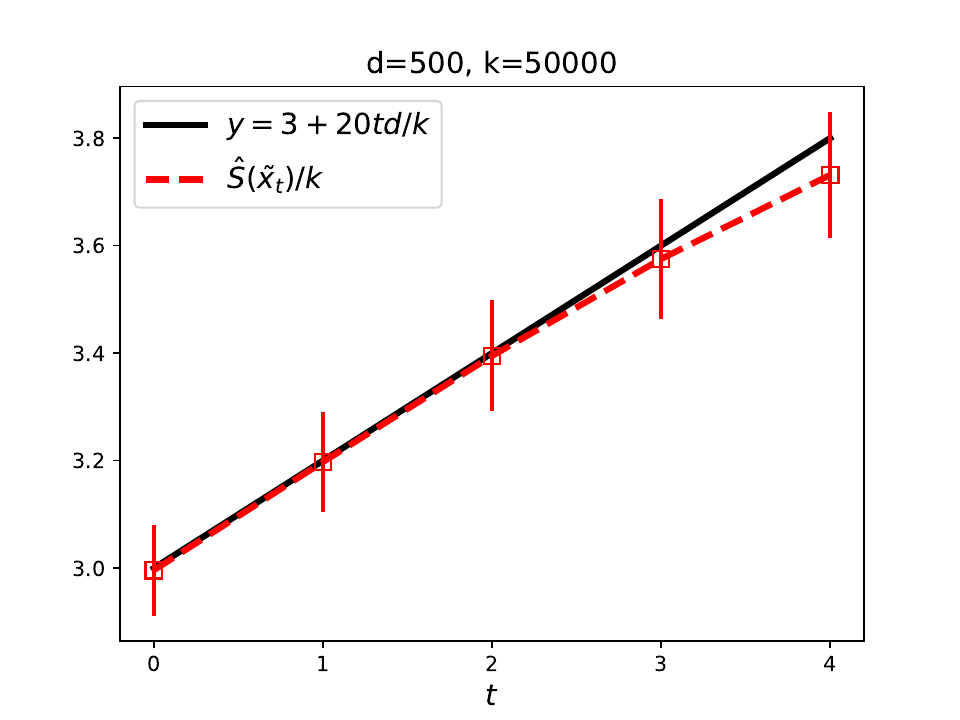}
\endminipage
\caption{Theoretical predictions of $\mathcal{S}$ along the power iteration paths versus the corresponding empirical values. Outcomes are averaged over 1000 independent experiments. The error bars reflect the intervals determined by two times the empirical standard deviation. }
\label{fig:increasing}
\end{figure}

\section{Conclusion}\label{sec:conclusion}
We analyze the dynamics of randomly initialized tensor power iteration in the overcomplete regime, using the Gaussian conditioning technique. We prove that it takes polynomially many iterates for the power method to find a true component of a random even-order tensor, thus refuting the previous conjecture that tensor power iteration converges in logarithmically many steps. We also establish that along the power iteration path, a popular polynomial objective function for tensor decomposition is strictly increasing for finitely many steps. Extensive numerical studies verify our theoretical results.

Our work leads to a number of fascinating open problems. First, it would be interesting to understand whether the power method indeed converges to a tensor component in polynomially many steps, within the regime where SOS-based methods succeed. One possible direction is to extend our analysis in Theorem~\ref{thm:increasing} to polynomially many iterations. Another appealing direction is to generalize our results to the case of odd-order tensors, for which the power iterates will no longer stay in a small neighborhood of the initialization.


\section*{Acknowledgements}
The authors would like to thank Tselil Schramm for suggesting this topic, as well as for many helpful conversations regarding the technical contents and the presentation of this work.
Y.W. and K.Z. were supported by the NSF through award DMS-2031883 and the Simons Foundation through Award 814639 for the Collaboration on the Theoretical Foundations of Deep Learning and by the NSF grant CCF-2006489.

\newpage
\bibliographystyle{alpha}
\bibliography{TPI.bib}

\newcommand{\etalchar}[1]{$^{#1}$}
\begin{thebibliography}{HAYWC19}

\bibitem[AGH{\etalchar{+}}14]{anandkumar2014tensor}
Animashree Anandkumar, Rong Ge, Daniel Hsu, Sham~M Kakade, and Matus Telgarsky.
\newblock Tensor decompositions for learning latent variable models.
\newblock {\em Journal of machine learning research}, 15:2773--2832, 2014.

\bibitem[AGJ14]{anandkumar2014guaranteed}
Animashree Anandkumar, Rong Ge, and Majid Janzamin.
\newblock Guaranteed non-orthogonal tensor decomposition via alternating
  rank-$1 $ updates.
\newblock {\em arXiv preprint arXiv:1402.5180}, 2014.

\bibitem[AGJ15]{anandkumar2015learning}
Animashree Anandkumar, Rong Ge, and Majid Janzamin.
\newblock Learning overcomplete latent variable models through tensor methods.
\newblock In {\em Conference on Learning Theory}, pages 36--112. PMLR, 2015.

\bibitem[AGJ17]{anandkumar2017analyzing}
Animashree Anandkumar, Rong Ge, and Majid Janzamin.
\newblock Analyzing tensor power method dynamics in overcomplete regime.
\newblock {\em Journal of Machine Learning Research}, 18(22):1--40, 2017.

\bibitem[BCPV19]{bhaskara2019smoothed}
Aditya Bhaskara, Aidao Chen, Aidan Perreault, and Aravindan Vijayaraghavan.
\newblock Smoothed analysis in unsupervised learning via decoupling.
\newblock In {\em 2019 IEEE 60th Annual Symposium on Foundations of Computer
  Science (FOCS)}, pages 582--610. IEEE, 2019.

\bibitem[BKS15]{barak2015dictionary}
Boaz Barak, Jonathan~A Kelner, and David Steurer.
\newblock Dictionary learning and tensor decomposition via the sum-of-squares
  method.
\newblock In {\em Proceedings of the forty-seventh annual ACM symposium on
  Theory of computing}, pages 143--151, 2015.

\bibitem[BM11]{bayati2011dynamics}
Mohsen Bayati and Andrea Montanari.
\newblock The dynamics of message passing on dense graphs, with applications to
  compressed sensing.
\newblock {\em IEEE Transactions on Information Theory}, 57(2):764--785, 2011.

\bibitem[CC70]{carroll1970analysis}
J~Douglas Carroll and Jih-Jie Chang.
\newblock Analysis of individual differences in multidimensional scaling via an
  n-way generalization of ``eckart-young'' decomposition.
\newblock {\em Psychometrika}, 35(3):283--319, 1970.

\bibitem[CJ10]{comon2010handbook}
Pierre Comon and Christian Jutten.
\newblock {\em Handbook of Blind Source Separation: Independent component
  analysis and applications}.
\newblock Academic press, 2010.

\bibitem[CMW20]{celentano2020estimation}
Michael Celentano, Andrea Montanari, and Yuchen Wu.
\newblock The estimation error of general first order methods.
\newblock In {\em Conference on Learning Theory}, pages 1078--1141. PMLR, 2020.

\bibitem[CS13]{cartwright2013number}
Dustin Cartwright and Bernd Sturmfels.
\newblock The number of eigenvalues of a tensor.
\newblock {\em Linear algebra and its applications}, 438(2):942--952, 2013.

\bibitem[DdL{\etalchar{+}}22]{ding2022fast}
Jingqiu Ding, Tommaso d'Orsi, Chih-Hung Liu, David Steurer, and Stefan Tiegel.
\newblock Fast algorithm for overcomplete order-3 tensor decomposition.
\newblock In {\em Conference on Learning Theory}, pages 3741--3799. PMLR, 2022.

\bibitem[DL06]{de2006link}
Lieven De~Lathauwer.
\newblock A link between the canonical decomposition in multilinear algebra and
  simultaneous matrix diagonalization.
\newblock {\em SIAM journal on Matrix Analysis and Applications},
  28(3):642--666, 2006.

\bibitem[DLCC07]{de2007fourth}
Lieven De~Lathauwer, Josphine Castaing, and Jean-Franois Cardoso.
\newblock Fourth-order cumulant-based blind identification of underdetermined
  mixtures.
\newblock {\em IEEE Transactions on Signal Processing}, 55(6):2965--2973, 2007.

\bibitem[DLDMV96]{de1996blind}
Lieven De~Lathauwer, Bart De~Moor, and Joos Vandewalle.
\newblock Blind source separation by simultaneous third-order tensor
  diagonalization.
\newblock In {\em 1996 8th European Signal Processing Conference (EUSIPCO
  1996)}, pages 1--4. IEEE, 1996.

\bibitem[GHK15]{ge2015learning}
Rong Ge, Qingqing Huang, and Sham~M Kakade.
\newblock Learning mixtures of gaussians in high dimensions.
\newblock In {\em Proceedings of the forty-seventh annual ACM symposium on
  Theory of computing}, pages 761--770, 2015.

\bibitem[GM15]{ge2015decomposing}
Rong Ge and Tengyu Ma.
\newblock Decomposing overcomplete 3rd order tensors using sum-of-squares
  algorithms.
\newblock {\em arXiv preprint arXiv:1504.05287}, 2015.

\bibitem[GM17]{ge2017optimization}
Rong Ge and Tengyu Ma.
\newblock On the optimization landscape of tensor decompositions.
\newblock {\em Advances in Neural Information Processing Systems}, 30, 2017.

\bibitem[H{\etalchar{+}}70]{harshman1970foundations}
Richard~A Harshman et~al.
\newblock Foundations of the parafac procedure: Models and conditions for an"
  explanatory" multimodal factor analysis.
\newblock 1970.

\bibitem[HAYWC19]{hao2019bootstrapping}
Botao Hao, Yasin Abbasi~Yadkori, Zheng Wen, and Guang Cheng.
\newblock Bootstrapping upper confidence bound.
\newblock {\em Advances in neural information processing systems}, 32, 2019.

\bibitem[HK13]{hsu2013learning}
Daniel Hsu and Sham~M Kakade.
\newblock Learning mixtures of spherical gaussians: moment methods and spectral
  decompositions.
\newblock In {\em Proceedings of the 4th conference on Innovations in
  Theoretical Computer Science}, pages 11--20, 2013.

\bibitem[HL13]{hillar2013most}
Christopher~J Hillar and Lek-Heng Lim.
\newblock Most tensor problems are np-hard.
\newblock {\em Journal of the ACM (JACM)}, 60(6):1--39, 2013.

\bibitem[HSS15]{hopkins2015tensor}
Samuel~B Hopkins, Jonathan Shi, and David Steurer.
\newblock Tensor principal component analysis via sum-of-square proofs.
\newblock In {\em Conference on Learning Theory}, pages 956--1006. PMLR, 2015.

\bibitem[HSS19]{hopkins2019robust}
Samuel~B Hopkins, Tselil Schramm, and Jonathan Shi.
\newblock A robust spectral algorithm for overcomplete tensor decomposition.
\newblock In {\em Conference on Learning Theory}, pages 1683--1722. PMLR, 2019.

\bibitem[HSSS16]{hopkins2016fast}
Samuel~B Hopkins, Tselil Schramm, Jonathan Shi, and David Steurer.
\newblock Fast spectral algorithms from sum-of-squares proofs: tensor
  decomposition and planted sparse vectors.
\newblock In {\em Proceedings of the forty-eighth annual ACM symposium on
  Theory of Computing}, pages 178--191, 2016.

\bibitem[Kie00]{kiers2000towards}
Henk~AL Kiers.
\newblock Towards a standardized notation and terminology in multiway analysis.
\newblock {\em Journal of Chemometrics: A Journal of the Chemometrics Society},
  14(3):105--122, 2000.

\bibitem[KKMP21]{kileel2021landscape}
Joe Kileel, Timo Klock, and Jo{\~a}o M~Pereira.
\newblock Landscape analysis of an improved power method for tensor
  decomposition.
\newblock {\em Advances in Neural Information Processing Systems},
  34:6253--6265, 2021.

\bibitem[KP19]{kileel2019subspace}
Joe Kileel and Joao~M Pereira.
\newblock Subspace power method for symmetric tensor decomposition and
  generalized pca.
\newblock {\em arXiv preprint arXiv:1912.04007}, 2019.

\bibitem[Kru77]{kruskal1977three}
Joseph~B Kruskal.
\newblock Three-way arrays: rank and uniqueness of trilinear decompositions,
  with application to arithmetic complexity and statistics.
\newblock {\em Linear algebra and its applications}, 18(2):95--138, 1977.

\bibitem[LW22]{li2022non}
Gen Li and Yuting Wei.
\newblock A non-asymptotic framework for approximate message passing in spiked
  models.
\newblock {\em arXiv preprint arXiv:2208.03313}, 2022.

\bibitem[MM19]{mondelli2019connection}
Marco Mondelli and Andrea Montanari.
\newblock On the connection between learning two-layer neural networks and
  tensor decomposition.
\newblock In {\em The 22nd International Conference on Artificial Intelligence
  and Statistics}, pages 1051--1060. PMLR, 2019.

\bibitem[MSS16]{ma2016polynomial}
Tengyu Ma, Jonathan Shi, and David Steurer.
\newblock Polynomial-time tensor decompositions with sum-of-squares.
\newblock In {\em 2016 IEEE 57th Annual Symposium on Foundations of Computer
  Science (FOCS)}, pages 438--446. IEEE, 2016.

\bibitem[MW22a]{montanari2022adversarial}
Andrea Montanari and Yuchen Wu.
\newblock Adversarial examples in random neural networks with general
  activations.
\newblock {\em arXiv preprint arXiv:2203.17209}, 2022.

\bibitem[MW22b]{montanari2022statistically}
Andrea Montanari and Yuchen Wu.
\newblock Statistically optimal first order algorithms: A proof via
  orthogonalization.
\newblock {\em arXiv preprint arXiv:2201.05101}, 2022.

\bibitem[NPOV15]{novikov2015tensorizing}
Alexander Novikov, Dmitrii Podoprikhin, Anton Osokin, and Dmitry~P Vetrov.
\newblock Tensorizing neural networks.
\newblock {\em Advances in neural information processing systems}, 28, 2015.

\bibitem[Pea94]{pearson1894contributions}
Karl Pearson.
\newblock Contributions to the mathematical theory of evolution.
\newblock {\em Philosophical Transactions of the Royal Society of London. A},
  185:71--110, 1894.

\bibitem[RV18]{rush2018finite}
Cynthia Rush and Ramji Venkataramanan.
\newblock Finite sample analysis of approximate message passing algorithms.
\newblock {\em IEEE Transactions on Information Theory}, 64(11):7264--7286,
  2018.

\bibitem[SS17]{schramm2017fast}
Tselil Schramm and David Steurer.
\newblock Fast and robust tensor decomposition with applications to dictionary
  learning.
\newblock In {\em Conference on Learning Theory}, pages 1760--1793. PMLR, 2017.

\bibitem[Ver18]{vershynin2018high}
Roman Vershynin.
\newblock {\em High-dimensional probability: An introduction with applications
  in data science}, volume~47.
\newblock Cambridge university press, 2018.

\bibitem[ZG01]{zhang2001rank}
Tong Zhang and Gene~H Golub.
\newblock Rank-one approximation to high order tensors.
\newblock {\em SIAM Journal on Matrix Analysis and Applications},
  23(2):534--550, 2001.

\end{thebibliography}

\newpage
\appendix

\section{Analysis of tensor power iteration: Proof of Theorem~\ref{thm:negative}}\label{sec:dynamics_TPI}

This section will de devoted to proving \cref{thm:negative}. For $t \in \NN_+$, we define
\begin{align*}
	& g_t = \sum_{i = 1}^{t} \Pi_{F_{1:i - 1}}^{\perp} \tA_i \txp_{i - 1} \cdot \frac{\langle \txp_{i - 1}, \tx_{t - 1} \rangle}{\|\txp_{i - 1}\|_2^2}, \\
	& v_t = \sum_{i = 2}^t \fp_{i - 1} \cdot \frac{\langle x_{i - 1}, \txp_{i - 1}\rangle}{\|\fp_{i - 1}\|_2^2} \cdot \frac{\langle \txp_{i - 1}, \tx_{t - 1} \rangle}{\|\txp_{i - 1}\|_2^2}.
\end{align*}
By \cref{lemma:conditioning} we have $y_t = g_t + v_t$. Applying Minkowski's inequality and power mean inequality, we deduce that $\|y_t\|_4^4 \leq 8(\|g_t\|_4^4 + \|v_t\|_4^4)$. Notice that if \cref{eq:main} does not hold, since $y_t = A\tx_{t - 1}$, then there exists $1 \leq t \leq T$ such that $\|y_t\|_4^4 \geq \veps^4 d^2$. As a result, in order to prove \cref{thm:negative}, it suffices to show with high probability, $\|y_t\|_4^4 < \veps^4 d^2$ for all $t \in [T]$. This further reduces to upper bounding $\|g_t\|_4^4$ and $\|v_t\|_4^4$. 

We first provide an upper bound for $\|g_t\|_4^4$. To this end, we define  
\begin{align}
	& w_t = \sum_{i = 1}^t \tA_i \txp_{i - 1} \cdot \frac{\langle \txp_{i - 1}, \tx_{t - 1} \rangle}{\|\txp_{i - 1}\|_2^2}, \label{eq:wt-def}\\
	& \eta_t = \sum_{i = 1}^t \sum_{j = 1}^{i - 1} \fp_j \cdot \frac{\langle \fp_j, \tA_i \txp_{i - 1} \rangle}{\|\fp_j\|_2^2} \cdot \frac{\langle \txp_{i - 1}, \tx_{t - 1} \rangle}{\|\txp_{i - 1}\|_2^2}. \label{eq:etat-def}
\end{align}
We immediately see that $g_t = w_t - \eta_t$, thus upper bounding $\|g_t\|_4^4$ can be achieved via upper bounding $\|w_t\|_4^4$ and $\|\eta_t\|_4^4$, respectively. As $\tA_i \perp \cF_{i-1,i-1}$, intuitively speaking, this suggests that $w_t$ behaves like a $k$-dimensional random vector with $\iid$ standard Gaussian entries. Therefore, with high probability $w_t$ has $\ell^4$-norm of order $k$. On the other hand, observe that $\eta_t$ is the sum of projections of random vectors onto low-dimensional subspaces, which only accounts for a small proportion of the total variation. As a result, we expect $\|\eta_t\|_4^4$ to be small.  

To make these heuristic arguments concrete, we establish the following two lemmas:
\begin{lem}\label{lemma:wt}
	Assume $T \ll k^{1/2}$. Then there exists a numerical constant $C > 0$, such that with probability $1 - o_d(1)$, for all $1 \leq t \leq T$ we have 
	\begin{align*}
		\|w_t\|_4^4 \leq Ck. 
	\end{align*}
\end{lem}
\begin{lem}\label{lemma:etat}
	Assume $T \ll k^{1/2}$. Then with probability $1 - o_d(1)$, for all $1 \leq t \leq T$ we have
	\begin{align*}
		\|\eta_t\|_2 \leq C T^{3/4} (\log k)^{1/4} \frac{\norm{\Pi_{x_0}^{\perp} \tx_{t - 1}}_2}{\sqrt{d}}.
	\end{align*}
\end{lem}
Note that $\|\eta_t\|_4 \leq \|\eta_t\|_2$. Invoking \cref{lemma:wt}, \cref{lemma:etat}, power mean inequality and Minkowski's inequality, we obtain that with high probability, for all $1 \leq t \leq T$,
\begin{align}\label{eq:gt4}
	\|g_t\|_4^4 \leq 8(\|w_t\|_4^4 + \|\eta_t\|_2^4) \leq C \left( k + \frac{\norm{\Pi_{x_0}^{\perp} \tx_{t - 1}}_2^2}{d^2} \cdot T^3 \log k \right) \le C (k + T^3 \log k) \ll d^2
\end{align}
for some numerical constant $C > 0$, since under the conditions of \cref{thm:negative} we have $T \ll d^{1/2}$. Therefore, in order to prove \cref{thm:negative}, it remains to upper bound $\|v_{t + 1}\|_4^4$. In what follows, we perform a crude analysis which uses the $\ell^2$-norm to control the $\ell^4$-norm. We comment that a more careful analysis might lead to an improved estimate. 

The next lemma establishes that the $\ell^2$-norm of $v_{t + 1}$ is close to the $\ell^2$-norm of $\Pi_{x_0}^{\perp}(\tx_t)$.
\begin{lem}\label{lemma:norm-of-vt}
	Under the condition of \cref{thm:negative}, with probability $1 - o_d(1)$, the following result holds for all $0 \leq t \leq T - 1$:
	\begin{align*}
		\left( 1 - \frac{C \log k}{\sqrt{d}} \right) \|\Pi_{x_0}^{\perp}(\tx_t)\|_2^2 \leq \|v_{t + 1}\|_2^2 \leq \left( 1 + \frac{C \log k}{\sqrt{d}} \right) \|\Pi_{x_0}^{\perp}(\tx_t)\|_2^2.
	\end{align*}
\end{lem}
The rest of the analysis is devoted to upper bounding $\|\Pi_{x_0}^{\perp}(\tx_t)\|_2^2$. By definition, we have
\begin{align}\label{eq:next}
	\|\Pi_{x_0}^{\perp}(\tx_t)\|_2^2 = \frac{d}{\|x_t\|_2^2} \cdot \| \Pi_{x_0}^{\perp}(x_t)\|_2^2.
\end{align}
Using \cref{eq:cond-xt} we see that
\begin{align*}
	& x_t =  \sum_{i = 0}^{t - 1} \txp_i \cdot \frac{\langle h_{i + 1}, f_t \rangle}{\|\txp_i\|_2^2} + \Pi_{X_{0:t - 1}}^{\perp} \bA_t^{\top} \fp_t,\\
	& \Pi_{x_0}^{\perp}(x_t) = \sum_{i = 1}^{t - 1} \txp_i \cdot \frac{\langle h_{i + 1}, f_t \rangle}{\|\txp_i\|_2^2} + \Pi_{X_{0:t - 1}}^{\perp} \bA_t^{\top} \fp_t.
\end{align*}
According to Pythagorean theorem,
\begin{align}
	& \|x_t\|_2^2 = \sum_{i = 0}^{t - 1} \frac{\langle h_{i + 1}, f_t \rangle^2}{\|\txp_{i}\|_2^2} + \| \Pi_{X_{0:t - 1}}^{\perp} \bA_t^{\top} \fp_t \|_2^2,  \label{eq:p1} \\
	& \|\Pi_{x_0}^{\perp}(x_t)\|_2^2 = \sum_{i = 1}^{t - 1} \frac{\langle h_{i + 1}, f_t \rangle^2}{\|\txp_{i}\|_2^2} + \| \Pi_{X_{0:t - 1}}^{\perp} \bA_t^{\top} \fp_t \|_2^2.  \label{eq:p2}
\end{align}
Using the definition of $h_{i + 1}$ given in \cref{eq:ht+1}, we deduce that
\begin{align}\label{eq:hi+1-decomp}
	\frac{\langle h_{i + 1}, f_t \rangle}{\|\txp_i\|_2} = \frac{\langle \fp_i, f_t \rangle}{\|\txp_i\|_2} \cdot \frac{\langle x_i, \txp_i \rangle}{\|\fp_i\|_2^2} + \frac{\langle \Pi_{F_{1:i}}^{\perp} \tA_{i + 1} \txp_i, f_t \rangle}{\|\txp_i\|_2}.
\end{align}
Combining \cref{eq:p1,eq:p2,eq:hi+1-decomp}, we obtain the following decomposition:
\begin{align}\label{eq:123}
\begin{split}
	& \|x_t\|_2^2 = \mathsf{I} + \mathsf{II} + \mathsf{III}, \\
	& \|\Pi_{x_0}^{\perp}(x_t)\|_2^2 = \mathsf{I}' + \mathsf{II} + \mathsf{III},
\end{split}
\end{align}
where
\begin{align}
	& \mathsf{I} = \sum_{i = 0}^{t - 1} \frac{\langle \Pi_{F_{1:i}}^{\perp} \tA_{i + 1} \txp_i, f_t \rangle^2}{\|\txp_i\|_2^2}, \qquad \mathsf{I}' = \sum_{i = 1}^{t - 1} \frac{\langle \Pi_{F_{1:i}}^{\perp} \tA_{i + 1} \txp_i, f_t \rangle^2}{\|\txp_i\|_2^2}, \label{eq:I-and-I'} \\
	& \mathsf{II} = \sum_{i = 1}^{t - 1} \frac{\langle \fp_i, f_t \rangle^2}{\|\fp_i\|_2^2} \cdot \frac{\langle x_i, \xp_i \rangle}{\|\fp_i\|_2^2} + \| \Pi_{X_{0:t - 1}}^{\perp} \bA_t^{\top} \fp_t \|_2^2, \label{eq:II} \\
	& \mathsf{III} = \sum_{i = 1}^{t - 1} \frac{2 \langle x_i, \txp_i \rangle \langle \fp_i, f_t \rangle \langle \Pi_{F_{1:i}}^{\perp} \tA_{i + 1} \txp_i, f_t \rangle}{\|\txp_i\|_2^2 \|\fp_i\|_2^2}. \label{eq:III}
\end{align}
Next, we will show that terms $\mathsf{II}$ and $\mathsf{III}$ are negligible comparing to terms $\mathsf{I}$ and $\mathsf{I}'$. Note that by Cauchy-Schwarz inequality, $|\mathsf{III}| \leq 2\sqrt{\mathsf{I}' \times \mathsf{II}}$, namely the term $\mathsf{III}$ is controlled by $\mathsf{II}$. This motivates us to first provide an upper bound for term $\mathsf{II}$. 
\begin{lem}\label{lemma:term-II}
	Under the condition of \cref{thm:negative}, with probability $1 - o_d(1)$,  for all $1 \leq t \leq T$ we have (note that $\mathsf{II}$ depends on $t$ as well)
	\begin{align*}
		\left( 1 - \frac{C \log k}{\sqrt{d}} \right)\|f_t\|_2^2 \leq \mathsf{II} \leq \left( 1 + \frac{C \log k}{\sqrt{d}} \right)\|f_t\|_2^2.
	\end{align*}
\end{lem}
The following lemma establishes an upper bound on $\|f_t\|_2^2$:
\begin{lem}\label{lemma:bound-ft}
	Under the condition of \cref{thm:negative}, there exists a numerical constant $C > 0$, such that with probability $1 - o_d(1)$, for all $1 \leq t \leq T$ we have
	\begin{align*}
		\norm{f_t}_2^2 \le C \left( k + \|\Pi_{x_0}^{\perp}(\tx_{t - 1})\|_2^6 \right).
	\end{align*}
\end{lem}
With the aid of \cref{lemma:term-II} and \cref{lemma:bound-ft}, we obtain that with high probability
\begin{align*}
	\mathsf{II} \leq C \left( k + \|\Pi_{x_0}^{\perp}(\tx_{t - 1})\|_2^6 \right)
\end{align*}
for all $1 \leq t \leq T$ and some positive numerical constant $C$. Next, we analyze terms $\mathsf{I}$ and $\mathsf{I}'$. To achieve this goal, we find the following lemma useful:
\begin{lem}\label{lemma:term-I}
	Under the condition of \cref{thm:negative}, there exists a numerical constant $C > 0$, such that with probability $1 - o_d(1)$, for all $0 \leq i < t \leq T$, we have
	\begin{align*}
		& \left| \frac{\langle \Pi_{F_{1:i}}^{\perp} \tA_{i + 1} \txp_i, f_t \rangle}{\|\txp_i\|_2} - \frac{3k}{d} \cdot \frac{\langle \txp_i, \tx_{t - 1} \rangle}{\|\txp_i\|_2} \right| \\
		 \leq\, & C \left( \sqrt{\frac{k T \log k}{d}} + \sqrt{\frac{T}{d}} \, \|\Pi_{x_0}^{\perp} (\tx_{t - 1})\|_2^3 + \sqrt{\frac{k \log k}{d}}\,  \|\Pi_{x_0}^{\perp} \tx_{t - 1}\|_2\right).
	\end{align*}
\end{lem}
Now we are in position to finish the proof of \cref{thm:negative}. For future convenience, we hereby establish a general framework for the analysis of tensor power iteration dynamics, based on \cref{lemma:wt} to \cref{lemma:term-I}. To begin with, let us denote
\begin{equation*}
	P_t = \norm{\Pi_{x_0}^{\perp} (\tx_t)}_2^2, \quad Q_t = \norm{\Pi_{x_0} (\tx_t)}_2^2.
\end{equation*}
Then, we know that $P_t + Q_t = \norm{\tx_t}_2^2 = d$, and that
\begin{equation*}
	\frac{P_t}{Q_t} = \frac{\norm{\Pi_{x_0}^{\perp} (x_t)}_2^2}{\norm{\Pi_{x_0} (x_t)}_2^2} = \frac{\mathsf{I}' + \mathsf{II} + \mathsf{III}}{\mathsf{I} - \mathsf{I}'}.
\end{equation*}
Recall that our aim is to show that $\norm{v_t}_2^4 \ll d^2$ for all $1 \le t \le T$. According to \cref{lemma:norm-of-vt}, this amounts to proving that $P_t \ll Q_t$ for all $t \in [T]$. Define the stopping time
\begin{equation*}
	T_{k} = \inf \{ t \in \NN_+ : P_t \ge k^{1/3} \}.
\end{equation*}
Since $k^{1/3} \ll d$, it then suffices to show that $T_{k} \ge T(k, d)$ with high probability, where $T(k, d)$ is defined in the statement of Theorem~\ref{thm:negative}.
\vspace{0.5em}

\noindent {\bf Step 1. A lower bound for $\mathsf{I} - \mathsf{I}'$.} By definition, we know that
\begin{equation*}
	\mathsf{I} - \mathsf{I}' = \frac{\langle \tA_1 \tx_0, f_t \rangle^2}{\norm{\tx_0}_2^2}.
\end{equation*}
Note that the proof of \cref{lemma:term-I} also implies that
\begin{align*}
	\left| \frac{\langle \tA_{1} \tx_0, f_t \rangle}{\|\tx_0\|_2} - \frac{3k}{d} \sqrt{Q_{t - 1}} \right| \le\, & C \left( \sqrt{\frac{k T \log k}{d}} + \sqrt{\frac{\log k}{d}} \, \|\Pi_{x_0}^{\perp} (\tx_{t - 1})\|_2^3 + \sqrt{\frac{k \log k}{d}}\,  \|\Pi_{x_0}^{\perp} \tx_{t - 1}\|_2\right) \\
	=\, & C \left( \sqrt{\frac{k T \log k}{d}} + \sqrt{\frac{\log k}{d}} \, P_{t-1}^{3/2} + \sqrt{\frac{k \log k}{d}}\,  P_{t-1}^{1/2} \right) \\
	=\, & C (\log k)^{1/2} \left( \sqrt{\frac{k T}{d}} + \frac{1}{\sqrt{d}} \, P_{t-1}^{3/2} + \sqrt{\frac{k}{d}}\,  P_{t-1}^{1/2} \right).
\end{align*}
For $t \le T_{k}$, we have $Q_{t - 1} \asymp d$, thus leading to the estimate:
\begin{align*}
	\mathsf{I} - \mathsf{I}' = \frac{\langle \tA_1 \tx_0, f_t \rangle^2}{\norm{\tx_0}_2^2} \ge\, & \frac{9k^2}{d^2} Q_{t-1} \cdot \left( 1 - \frac{C (\log k)^{1/2} \left( \sqrt{dk T} + \sqrt{d} \, P_{t-1}^{3/2} + \sqrt{dk}\,  P_{t-1}^{1/2} \right)}{3k \sqrt{Q_{t - 1}}} \right)^2 \\
	\ge\, & \frac{9k^2}{d^2} Q_{t-1} \cdot \left( 1 - C (\log k)^{1/2} \left( k^{-1/2}\, T^{1/2} + k^{-1} P_{t-1}^{3/2} + k^{-1/2}\,  P_{t-1}^{1/2} \right) \right)^2.
\end{align*}
Hence, it follows that
\begin{equation*}
	\frac{1}{\mathsf{I} - \mathsf{I}'} \le \frac{d^2}{9k^2 Q_{t-1}} \cdot \left( 1 + C (\log k)^{1/2} \left( k^{-1/2}\, T^{1/2} + k^{-1} P_{t-1}^{3/2} + k^{-1/2}\,  P_{t-1}^{1/2} \right) \right),
\end{equation*}
where the last line is due to the fact that 
\begin{align*}
	(\log k)^{1/2} \left( k^{-1/2}\, T^{1/2} + k^{-1} P_{t-1}^{3/2} + k^{-1/2}\,  P_{t-1}^{1/2} \right) = o(1).
\end{align*}
Since by our assumption, $T \ll k^{1/3}$, and $P_{t - 1} \leq k^{1/3}$.

\vspace{0.5em}

\noindent {\bf Step 2. An upper bound for $\mathsf{I}' + \mathsf{II} + \mathsf{III}$.}
Using Cauchy-Schwarz inequality, we get
\begin{equation*}
	\mathsf{I}' + \mathsf{II} + \mathsf{III} \le \mathsf{I}' + \mathsf{II} + 2 \sqrt{\mathsf{I}' \times \mathsf{II}} = \left( \sqrt{\mathsf{I}'} + \sqrt{\mathsf{II}} \right)^2.
\end{equation*}
To upper bound $\sqrt{\mathsf{I}'}$, we note that
\begin{align*}
	& \left\vert \sqrt{\mathsf{I}'} - \frac{3k}{d} \sqrt{P_{t-1}} \right\vert = \frac{\vert \mathsf{I}' - 9k^2 P_{t-1} / d^2 \vert}{\sqrt{\mathsf{I}'} + 3k \sqrt{P_{t-1}} / d} \\
	\le\, & \frac{1}{\sqrt{\mathsf{I}'} + 3k \sqrt{P_{t-1}} / d} \cdot \left\vert \sum_{i = 1}^{t - 1} \frac{\langle \Pi_{F_{1:i}}^{\perp} \tA_{i + 1} \txp_i, f_t \rangle^2}{\|\txp_i\|_2^2} - \frac{9k^2}{d^2} \sum_{i=1}^{t-1} \frac{\langle \txp_i, \tx_{t - 1} \rangle^2}{\|\txp_i\|_2^2} \right\vert \\
	=\, & \frac{1}{\sqrt{\mathsf{I}'} + 3k \sqrt{P_{t-1}} / d} \cdot \left\vert \sum_{i = 1}^{t - 1} \left( \frac{\langle \Pi_{F_{1:i}}^{\perp} \tA_{i + 1} \txp_i, f_t \rangle}{\|\txp_i\|_2} - \frac{3k}{d} \frac{\langle \txp_i, \tx_{t - 1} \rangle}{\|\txp_i\|_2} \right) \left( \frac{\langle \Pi_{F_{1:i}}^{\perp} \tA_{i + 1} \txp_i, f_t \rangle}{\|\txp_i\|_2} + \frac{3k}{d} \frac{\langle \txp_i, \tx_{t - 1} \rangle}{\|\txp_i\|_2} \right) \right\vert \\
	\le\, & \frac{1}{\sqrt{\mathsf{I}'} + 3k \sqrt{P_{t-1}} / d} \cdot \sqrt{\sum_{i = 1}^{t - 1} \left( \frac{\langle \Pi_{F_{1:i}}^{\perp} \tA_{i + 1} \txp_i, f_t \rangle}{\|\txp_i\|_2} - \frac{3k}{d} \frac{\langle \txp_i, \tx_{t - 1} \rangle}{\|\txp_i\|_2} \right)^2} \\
	& \times \sqrt{\sum_{i = 1}^{t - 1} \left( \frac{\langle \Pi_{F_{1:i}}^{\perp} \tA_{i + 1} \txp_i, f_t \rangle}{\|\txp_i\|_2} + \frac{3k}{d} \frac{\langle \txp_i, \tx_{t - 1} \rangle}{\|\txp_i\|_2} \right)^2} \\
	\stackrel{(i)}{\le}\, & C  \sqrt{\sum_{i = 1}^{t - 1} \left( \frac{\langle \Pi_{F_{1:i}}^{\perp} \tA_{i + 1} \txp_i, f_t \rangle}{\|\txp_i\|_2} - \frac{3k}{d} \frac{\langle \txp_i, \tx_{t - 1} \rangle}{\|\txp_i\|_2} \right)^2} \\
	\stackrel{(ii)}{\le}\, & C \left( \sqrt{\frac{k \log k}{d}}\, T + \sqrt{\frac{T^2}{d}} \, \|\Pi_{x_0}^{\perp} (\tx_{t - 1})\|_2^3 + \sqrt{\frac{T k \log k}{d}}\,  \|\Pi_{x_0}^{\perp} \tx_{t - 1}\|_2\right) \\
	\le\, & C (\log k)^{1/2} \left( \sqrt{\frac{k}{d}} \, T + \frac{T}{\sqrt{d}} \, P_{t-1}^{3/2} + \sqrt{\frac{T k}{d}}\, P_{t-1}^{1/2} \right),
\end{align*}
where $(i)$ follows from Minkowski's inequality, and $(ii)$ follows from \cref{lemma:term-I}. Using \cref{lemma:term-II} and \cref{lemma:bound-ft}, we obtain that with high probability,
\begin{equation*}
	\sqrt{\mathsf{II}} \le C \norm{f_t}_2 \le C \left( \sqrt{k} + \|\Pi_{x_0}^{\perp} (\tx_{t - 1})\|_2^3 \right) = C \left( \sqrt{k} + P_{t-1}^{3/2} \right),
\end{equation*}
thus leading to the estimate:
\begin{equation*}
	\sqrt{\mathsf{I}'} + \sqrt{\mathsf{II}} \le \frac{3k}{d} \sqrt{P_{t-1}} + C (\log k)^{1/2} \left( \sqrt{\frac{k}{d}}\, T + \sqrt{\frac{T k}{d}}\, P_{t-1}^{1/2} \right) + C \left( \sqrt{k} + P_{t-1}^{3/2} \right),
\end{equation*}
since $T \ll d^{1/3}$ by our assumption. We finally obtain that
\begin{align*}
	\mathsf{I}' + \mathsf{II} + \mathsf{III} \le\, & \left( \frac{3k}{d} \sqrt{P_{t-1}} + C (\log k)^{1/2} \left( \sqrt{\frac{k}{d}}\, T + \sqrt{\frac{T k}{d}}\, P_{t-1}^{1/2} \right) + C \left( \sqrt{k} + P_{t-1}^{3/2} \right) \right)^2 \\
	\le\, & \left( \frac{3k}{d} \sqrt{P_{t-1}} + C (\log k)^{1/2} \sqrt{\frac{T k}{d}}\, P_{t-1}^{1/2} + C \left( \sqrt{k} + P_{t-1}^{3/2} \right) \right)^2.
\end{align*}
\vspace{0.5em}

\noindent {\bf Step 3. Write a recurrence inequality for $P_t / Q_t$.}
Combining our results from the previous steps gives the following recurrence relationship:
\begin{align*}
	\frac{P_t}{Q_t} = \frac{\mathsf{I}' + \mathsf{II} + \mathsf{III}}{\mathsf{I} - \mathsf{I}'} \le\, & \frac{d^2}{9 k^2 Q_{t-1}} \cdot \left( \frac{3k}{d} \sqrt{P_{t-1}} + C (\log k)^{1/2} \sqrt{\frac{T k}{d}}\, P_{t-1}^{1/2} + C \left( \sqrt{k} + P_{t-1}^{3/2} \right) \right)^2 \\
	& \times \left( 1 + C (\log k)^{1/2} \left( k^{-1/2}\, T^{1/2} + k^{-1} P_{t-1}^{3/2} + k^{-1/2}\,  P_{t-1}^{1/2} \right) \right).
\end{align*}
Denote the right hand side of the above inequality as $U_{k, d, T} (P_{t - 1}/Q_{t-1})$, then we know that $U_{k, d, T}$ is an increasing function: As $P_{t-1}/Q_{t-1}$ increases, $P_{t-1}$ increases and $Q_{t-1}$ decreases. Hence, the right hand side of the above inequality will also increase. As a consequence, we deduce that
\begin{equation*}
	\frac{P_T}{Q_T} \le U_{k, d, T}^T \left( \frac{P_0}{Q_0} \right) = U_{k, d, T}^T \left( 0 \right) \le U_{k, d, T}^T \left( \frac{k^{1/3}/2}{d - k^{1/3}/2} \right).
\end{equation*}
By definition of $U_{k, d, T}$, whenever
\begin{equation*}
	\frac{k^{1/3}/2}{d - k^{1/3}/2} \le x \le \frac{k^{1/3}}{d - k^{1/3}},
\end{equation*}
we have the following estimate:
\begin{align*}
	U_{k, d, T} \left( x \right) \le\, & x \times \left( 1 + 2 C \cdot \frac{d}{k^{2/3}} + C (\log k)^{1/2} \cdot \frac{d^{1/2} T^{1/2}}{k^{1/2}} \right)^2 \\
	& \times \left( 1 + C (\log k)^{1/2} \left( \frac{T^{1/2}}{k^{1/2}} + 2 k^{-1/3} \right) \right) \\
	\stackrel{(i)}{\le}\, & \left( 1 + \frac{C d}{k^{2/3}} + C (\log k)^{1/2} \left( \frac{d^{1/2} T^{1/2}}{k^{1/2}} + k^{-1/3} \right) \right) \cdot x \\
	\stackrel{(ii)}{\le}\, & \left( 1 + \frac{C d}{k^{2/3}} + C (\log k)^{1/2} \cdot \frac{d^{1/2} T^{1/2}}{k^{1/2}} \right) \cdot x,
\end{align*}
where $(i)$ and $(ii)$ both follow from our assumption: $d^{3/2} \ll k \ll d^2$. Since $k^{1/3} \ll d$, it suffices to show that for $T = T(k, d) \ll (\log k)^{-1/3} \cdot (k^{2/3} / d)$,
\begin{equation*}
	U_{k, d, T}^T \left( \frac{k^{1/3}/2}{d - k^{1/3}/2} \right) \le \frac{k^{1/3}}{d - k^{1/3}}.
\end{equation*}
Let $T'$ be the maximum integer such that
\begin{equation*}
	U_{k, d, T}^{T'} \left( \frac{k^{1/3}/2}{d - k^{1/3}/2} \right) \le \frac{k^{1/3}}{d - k^{1/3}},
\end{equation*}
then by monotonicity of $U_{k, d, T}$ and maximality of $T'$ we know that
\begin{align*}
	\frac{k^{1/3}}{d - k^{1/3}} < U_{k, d, T}^{T' + 1} \left( \frac{k^{1/3}/2}{d - k^{1/3}/2} \right) \le\, & \left( 1 + \frac{C d}{k^{2/3}} + C (\log k)^{1/2} \cdot \frac{d^{1/2} T^{1/2}}{k^{1/2}} \right)^{T' + 1} \times \frac{k^{1/3}/2}{d - k^{1/3}/2} \\
	\le\, & \exp \left( C (T' + 1) \left( \frac{d}{k^{2/3}} + (\log k)^{1/2} \cdot \frac{d^{1/2} T^{1/2}}{k^{1/2}} \right) \right) \times \frac{k^{1/3}/2}{d - k^{1/3}/2},
\end{align*}
which further implies that
\begin{align*}
	T' + 1 \ge\, & (\log 2 / C) \cdot \left( \frac{d}{k^{2/3}} + (\log k)^{1/2} \cdot \frac{d^{1/2} T^{1/2}}{k^{1/2}} \right)^{-1} \ge \frac{\log 2}{2 C} \cdot \min \left\{ \frac{k^{2/3}}{d}, \ (\log k)^{-1/3} k^{1/6} \right\} \\
	\ge\, & \frac{\log 2}{2 C} \cdot (\log k)^{-1/3} \cdot \frac{k^{2/3}}{d} \implies T' \gtrsim (\log k)^{-1/3} \cdot \frac{k^{2/3}}{d}.
\end{align*}
This completes the proof of \cref{thm:negative}.

\subsection{Proof of Lemma~\ref{lemma:wt}}
For $i, t \in [T]$, we define 
	\begin{align}\label{eq:z-and-alpha}
		z_i = \tA_i \txp_{i - 1} \cdot \frac{\sqrt{d}}{\|\txp_{i - 1}\|_2}, \qquad \alpha_{i, t} = \frac{\langle \txp_{i - 1}, \tx_{t - 1} \rangle}{\sqrt{d} \|\txp_{i - 1}\|_2}. 
	\end{align}
	We immediately see that for all $t \in [T]$, we have $\sum_{i = 1}^t \alpha_{i,t}^2 = 1$, and that $w_t = \sum_{i=1}^{t} \alpha_{i, t} z_i$. According to \cref{lemma:conditioning}, for all $i \in [T]$ we have $\tA_i \perp \sigma(\cF_{i-1, i-1} \cup \sigma (\bA_1, \cdots, \bA_{i-1}, \tA_1, \cdots, \tA_{i - 1}))$,  we then obtain that for any $i \in [T]$ and $\{z_1, \cdots, z_{i - 1}\}$, $z_i \mid z_1, \cdots, z_{i - 1} \overset{d}{=} \normal(0, I_k)$. As a result, we see that $z_1, \cdots, z_T$ are independent and identically distributed random vectors with marginal distribution $\normal(0, I_k)$. Hence, it suffices to prove that
	\begin{equation*}
		\sup_{\alpha \in \S^{T - 1}} \norm{\sum_{i = 1}^{T} \alpha_i z_i}_4^4 \le C k 
	\end{equation*}
	with high probability. To this end, we use a covering argument. Fix $\veps \in (0, 1)$ (to be determined later), let $N_\veps (\S^{T - 1})$ be an $\veps$-covering of $\S^{T - 1}$. Then for any $\alpha \in \S^{T - 1}$, there exists $\alpha' \in N_\veps (\S^{T - 1})$ such that $\norm{\alpha - \alpha'}_2 \le \veps$, thus leading to
	\begin{equation*}
		\norm{\sum_{i = 1}^{T} \alpha_i z_i}_4 \le \norm{\sum_{i = 1}^{T} \alpha'_i z_i}_4 + \norm{\sum_{i = 1}^{T} ( \alpha_i - \alpha'_i) z_i}_4 \le \sup_{\alpha \in N_{\veps} (\S^{T - 1})} \norm{\sum_{i = 1}^{T} \alpha_i z_i}_4 + \veps \cdot \sup_{\alpha \in \S^{T - 1}} \norm{\sum_{i = 1}^{T} \alpha_i z_i}_4,
	\end{equation*}
    which further implies that
    \begin{align*}
    	& \sup_{\alpha \in \S^{T - 1}} \norm{\sum_{i = 1}^{T} \alpha_i z_i}_4 \le \sup_{\alpha \in N_{\veps} (\S^{T - 1})} \norm{\sum_{i = 1}^{T} \alpha_i z_i}_4 + \veps \cdot \sup_{\alpha \in \S^{T - 1}} \norm{\sum_{i = 1}^{T} \alpha_i z_i}_4 \\
    	\implies\, & \sup_{\alpha \in \S^{T - 1}} \norm{\sum_{i = 1}^{T} \alpha_i z_i}_4 \le \frac{1}{1 - \veps} \cdot \sup_{\alpha \in N_{\veps} (\S^{T - 1})} \norm{\sum_{i = 1}^{T} \alpha_i z_i}_4.
    \end{align*}
    Now, for any fixed $\alpha \in N_{\veps} (\S^{T - 1})$, we know that $\sum_{i = 1}^{T} \alpha_i z_i \sim \sN (0, I_k)$. According to Lemma~\ref{lemma:conc_non_exp}, we know that there exists a constant $C > 0$ such that
    \begin{equation*}
    	\P \left( \norm{\sum_{i = 1}^{T} \alpha_i z_i}_4^4 \ge Ck \right) \le C \exp (- C k^{1/2}).
    \end{equation*}
    Applying a union bound then gives
    \begin{equation*}
    	\P \left( \sup_{\alpha \in N_{\veps} (\S^{T - 1})} \norm{\sum_{i = 1}^{T} \alpha_i z_i}_4^4 \ge C k \right) \le C \left( \frac{C}{\veps} \right)^T \exp (- C k^{1/2}).
    \end{equation*}
    By our assumption, $T \ll k^{1/2}$. Now we choose $\veps = 1/2$, it follows that
    \begin{equation*}
    	\sup_{\alpha \in \S^{T - 1}} \norm{\sum_{i=1}^{T} \alpha_i z_i}_4^4 \le 2^4 \cdot \sup_{\alpha \in N_{\veps} (\S^{T - 1})} \norm{\sum_{i = 1}^{T} \alpha_i z_i}_4^4 \le C k
    \end{equation*}
     with high probability. This completes the proof of Lemma~\ref{lemma:wt}.
     
\subsection{Proof of Lemma~\ref{lemma:etat}}
Recall the definition of $\eta_t$ from \cref{eq:etat-def}:
	\begin{equation*}
		\eta_t = \sum_{i = 1}^t \sum_{j = 1}^{i - 1} \fp_j \cdot \frac{\langle \fp_j, \tA_i \txp_{i - 1} \rangle}{\|\fp_j\|_2^2} \cdot \frac{\langle \txp_{i - 1}, \tx_{t - 1} \rangle}{\|\txp_{i - 1}\|_2^2} = \sum_{i=1}^{t} \alpha_{i, t} \Pi_{F_{1:i-1}} z_i = \sum_{i=2}^{t} \alpha_{i, t} \Pi_{F_{1:i-1}} z_i,
	\end{equation*}
    where the last equality follows from the fact that $\Pi_{F_{1:0}} = 0$, and the $\alpha_{i, t}$'s and $z_i$'s are defined in the proof of \cref{lemma:wt}. We thus obtain that
    \begin{align*}
    	& \norm{\eta_t}_2^2 = \sum_{i=2}^{t} \sum_{j=2}^{t} \alpha_{i, t} \alpha_{j, t} \left\langle \Pi_{F_{1:i-1}} z_i, \Pi_{F_{1:j-1}} z_j \right\rangle \\
    	\le\, & \sum_{i=2}^{t} \alpha_{i, t}^2 \norm{\Pi_{F_{1:i-1}} z_i}_2^2 + 2 \sum_{2 \le i < j \le t} \left\vert \alpha_{i, t} \alpha_{j, t} \left\langle \Pi_{F_{1:i-1}} z_i, \Pi_{F_{1:j-1}} z_j \right\rangle \right\vert,
    \end{align*}
    where for $i < j$, we know that
    \begin{equation*}
    	\left\langle \Pi_{F_{1:i-1}} z_i, \Pi_{F_{1:j-1}} z_j \right\rangle = z_j^\top \Pi_{F_{1:j-1}} \Pi_{F_{1:i-1}} z_i = z_j^\top \Pi_{F_{1:i-1}} z_i = \Tr \left( \Pi_{F_{1:i-1}} z_i z_j^\top \right).
    \end{equation*}
    Using the same argument as in the proof of \cref{lemma:wt}, we see that for $i < j$, $\Pi_{F_{1:i-1}}$, $z_i \sim \sN (0, I_k)$, and $z_j \sim \sN (0, I_k)$ are mutually independent. In fact, given $(\Pi_{F_{1:i-1}}, z_i)$, the conditional distribution of $z_j$ is always $\sN (0, I_k)$, and given $\Pi_{F_{1:i-1}}$, the conditional distribution of $z_i$ is always $\sN (0, I_k)$. This further implies that
    \begin{align*}
    	& \P \left( \left\vert z_j^\top \Pi_{F_{1:i-1}} z_i \right\vert \ge \sqrt{C \log k} \norm{\Pi_{F_{1:i-1}} z_i}_2 \big\vert (\Pi_{F_{1:i-1}}, z_i) \right) \le\, k^{- C}, \\
    	& \P \left( \norm{\Pi_{F_{1:i-1}} z_i}_2 \ge \sqrt{C T} \right) \le\,  C \exp (- C T),
    \end{align*}
    where $C > 0$ is an absolute constant. Therefore, we conclude that with probability $1 - o_d (1)$, for all $2 \le i < j \le T$, one has
    \begin{equation*}
    	\left\vert \left\langle \Pi_{F_{1:i-1}} z_i, \Pi_{F_{1:j-1}} z_j \right\rangle \right\vert \le \sqrt{C T \log k}, \ \norm{\Pi_{F_{1:i-1}} z_i}_2^2 \le CT,
    \end{equation*}
    thus leading to the following estimate:
    \begin{align*}
    	\norm{\eta_t}_2^2 \le\, & CT \sum_{i=2}^{t} \alpha_{i, t}^2 + 2 \sqrt{CT \log k} \sum_{2 \le i < j \le t} \left\vert \alpha_{i, t} \alpha_{j, t} \right\vert \le CT \left( 1 - \alpha_{1, t}^2 \right) + \sqrt{CT \log k} \left( \sum_{i=2}^{t} \left\vert \alpha_{i, t} \right\vert \right)^2 \\
    	\le\, & CT \left( 1 - \alpha_{1, t}^2 \right) + t \sqrt{C T \log k} \cdot \sum_{i=2}^{t} \alpha_{i, t}^2 \le C T^{3/2} (\log k)^{1/2} \left( 1 - \alpha_{1, t}^2 \right).
    \end{align*}
    Note that by definition, we have
    \begin{equation*}
    	1 - \alpha_{1, t}^2 = 1 - \frac{\langle \tx_{0}^{\perp}, \tx_{t-1} \rangle^2}{d \norm{\tx_{0}^{\perp}}_2^2} = 1 - \frac{1}{d} \norm{\Pi_{x_0} \tx_{t-1}}_2^2 = \frac{1}{d} \norm{\Pi_{x_0}^{\perp} \tx_{t - 1}}_2^2.
    \end{equation*}
    Hence, we finally deduce that with high probability,
    \begin{equation*}
    	\norm{\eta_t}_2 \le C T^{3/4} (\log k)^{1/4} \frac{\norm{\Pi_{x_0}^{\perp} \tx_{t - 1}}_2}{\sqrt{d}}
    \end{equation*}
    for all $t \in [T]$, as desired. This concludes the proof.
    
\subsection{Proof of Lemma~\ref{lemma:norm-of-vt}}
Recall the definition of $v_{t+1}$:
    \begin{equation*}
    	v_{t+1} = \sum_{i = 1}^{t} f_i^{\perp} \cdot \frac{\langle x_i,  \tilde{x}_i^{\perp} \rangle}{\norm{f_i^{\perp}}_2^2} \cdot \frac{\langle \tilde{x}_i^{\perp}, \tilde{x}_t \rangle}{\norm{\tilde{x}_i^{\perp}}_2^2}.
    \end{equation*}
    Since $\{ f_i^{\perp} \}_{1 \le i \le t}$ is an orthogonal set, we readily see that
    \begin{align*}
    	\norm{v_{t+1}}_2^2 =\, & \sum_{i = 1}^{t} \frac{\langle x_i,  \tilde{x}_i^{\perp} \rangle^2}{\norm{f_i^{\perp}}_2^2} \cdot \frac{\langle \tilde{x}_i^{\perp}, \tilde{x}_t \rangle^2}{\norm{\tilde{x}_i^{\perp}}_2^4} = \sum_{i = 1}^{t} \frac{\norm{x_i}_2^2}{d} \cdot \frac{\langle \tilde{x}_i,  \tilde{x}_i^{\perp} \rangle^2}{\norm{f_i^{\perp}}_2^2} \cdot \frac{\langle \tilde{x}_i^{\perp}, \tilde{x}_t \rangle^2}{\norm{\tilde{x}_i^{\perp}}_2^4} \\
    	=\, & \sum_{i = 1}^{t} \frac{\norm{x_i}_2^2}{d} \cdot \frac{\langle \tilde{x}_i^{\perp},  \tilde{x}_i^{\perp} \rangle^2}{\norm{f_i^{\perp}}_2^2} \cdot \frac{\langle \tilde{x}_i^{\perp}, \tilde{x}_t \rangle^2}{\norm{\tilde{x}_i^{\perp}}_2^4} = \sum_{i = 1}^{t} \frac{\norm{x_i}_2^2}{d} \cdot \frac{\langle \tilde{x}_i^{\perp}, \tilde{x}_t \rangle^2}{\norm{f_i^{\perp}}_2^2} = \sum_{i = 1}^{t} \frac{\norm{x_i^{\perp}}_2^2}{\norm{f_i^{\perp}}_2^2} \cdot \frac{\langle \tilde{x}_i^{\perp}, \tilde{x}_t \rangle^2}{\norm{\tilde{x}_i^{\perp}}_2^2}.
    \end{align*}
    According to \cref{lemma:conditioning}, we have $x_i^{\perp} = \Pi_{X_{0: i-1}}^{\perp} \bar{A}_i^{\top} f_i^{\perp}$ where $\bar{A}_i \perp (\Pi_{X_{0: i-1}}^{\perp}, f_i^{\perp})$. Therefore, given $(\Pi_{X_{0: i-1}}^{\perp}, f_i^{\perp})$, the conditional distribution of $x_i^{\perp}$ is specified as
    \begin{equation*}
    	x_i^{\perp} \stackrel{d}{=} \frac{\norm{f_i^{\perp}}_2}{\sqrt{d}} \cdot \sN \left( 0, \Pi_{X_{0: i-1}}^{\perp} \right) \Big\vert \left( \Pi_{X_{0: i-1}}^{\perp}, f_i^{\perp} \right),
    \end{equation*}
    which further implies that $\norm{x_i^{\perp}}_2^2 / \norm{f_i^{\perp}}_2^2 \sim \chi^2 (d - i) / d$. Using standard concentration arguments, we know that with high probability for all $i \in [T]$:
    \begin{equation*}
    	\left\vert \frac{\norm{x_i^{\perp}}_2^2}{\norm{f_i^{\perp}}_2^2} - 1 \right\vert \le C \left( \frac{T}{d} + \frac{\log k}{\sqrt{d}} \right) \le \frac{C \log k}{\sqrt{d}},
    \end{equation*}
    where the last inequality follows from the condition of \cref{thm:negative}: $T \ll \sqrt{d}$. With the aid of the above estimation, we deduce that
    \begin{align*}
    	\left\vert \|v_{t + 1}\|_2^2 - \|\Pi_{x_0}^{\perp}(\tx_t)\|_2^2 \right\vert =\, & \left\vert \sum_{i = 1}^{t} \left( \frac{\norm{x_i^{\perp}}_2^2}{\norm{f_i^{\perp}}_2^2} - 1 \right) \frac{\langle \tilde{x}_i^{\perp}, \tilde{x}_t \rangle^2}{\norm{\tilde{x}_i^{\perp}}_2^2} \right\vert \le \sum_{i = 1}^{t} \left\vert \frac{\norm{x_i^{\perp}}_2^2}{\norm{f_i^{\perp}}_2^2} - 1 \right\vert \cdot \frac{\langle \tilde{x}_i^{\perp}, \tilde{x}_t \rangle^2}{\norm{\tilde{x}_i^{\perp}}_2^2} \\
    	\le\, & \frac{C \log k}{\sqrt{d}} \sum_{i = 1}^{t} \frac{\langle \tilde{x}_i^{\perp}, \tilde{x}_t \rangle^2}{\norm{\tilde{x}_i^{\perp}}_2^2} = \frac{C \log k}{\sqrt{d}} \|\Pi_{x_0}^{\perp}(\tx_t)\|_2^2,
    \end{align*}
    which completes the proof of this lemma.
    
\subsection{Proof of Lemma~\ref{lemma:term-II}}
The proof is similar to that of \cref{lemma:norm-of-vt}. By definition, $x_t^{\perp} = \Pi_{X_{0:t - 1}}^{\perp} \bA_t^{\top} \fp_t$, we get that
	\begin{equation*}
		\mathsf{II} = \sum_{i=1}^{t} \frac{\langle \fp_i, f_t \rangle^2}{\|\fp_i\|_2^2} \cdot \frac{\langle x_i, \xp_i \rangle}{\|\fp_i\|_2^2} = \sum_{i=1}^{t} \frac{\langle \fp_i, f_t \rangle^2}{\|\fp_i\|_2^2} \cdot \frac{\norm{\xp_i}_2^2}{\|\fp_i\|_2^2}.
	\end{equation*}
    From the proof of \cref{lemma:norm-of-vt} we know that, with probability $1 - o_d (1)$, for all $i \in [T]$ one has
    \begin{equation*}
    	\frac{\norm{\xp_i}_2^2}{\|\fp_i\|_2^2} \in \left[ 1 - \frac{C \log k}{\sqrt{d}}, 1 + \frac{C \log k}{\sqrt{d}} \right],
    \end{equation*}
    which immediately implies that
    \begin{align*}
    	\mathsf{II} \le\, & \left( 1 + \frac{C \log k}{\sqrt{d}} \right) \sum_{i=1}^{t} \frac{\langle \fp_i, f_t \rangle^2}{\|\fp_i\|_2^2} = \left( 1 + \frac{C \log k}{\sqrt{d}} \right) \norm{f_t}_2^2, \\
    	\mathsf{II} \ge\, & \left( 1 - \frac{C \log k}{\sqrt{d}} \right) \sum_{i=1}^{t} \frac{\langle \fp_i, f_t \rangle^2}{\|\fp_i\|_2^2} = \left( 1 - \frac{C \log k}{\sqrt{d}} \right) \norm{f_t}_2^2.
    \end{align*}
    This completes the proof.

\subsection{Proof of Lemma~\ref{lemma:bound-ft}}
Note that $y_t = w_t - \eta_t + v_t$, then by power mean inequality and Minkowski's inequality,
	\begin{align*}
		\|f_t\|_2^2 = \norm{y_t}_6^6 = \|w_t - \eta_t + v_t\|_6^6 \leq 243 \cdot (\|w_t\|_6^6 + \|\eta_t\|_6^6 + \|v_t\|_6^6). 
	\end{align*}
	It follows from \cref{lemma:etat} that with high probability,
	\begin{equation*}
		\|\eta_t\|_6^6 \leq \|\eta_t\|_2^6 \leq C T^{9/2} (\log k)^{3/2} \frac{\norm{\Pi_{x_0}^{\perp} \tx_{t-1}}_2^6}{d^3}
	\end{equation*}
	for all $1 \leq t \leq T$. Leveraging \cref{lemma:norm-of-vt}, we know that with high probability, for all $1 \leq t \leq T$ we have $\|v_t\|_6^6 \leq \|v_t\|_2^6 \leq 2\|\Pi_{x_0}^{\perp}(\tx_{t - 1})\|_2^6$. Finally, we upper bound $\|w_t\|_6^6$. Recall that $z_i$ and $\alpha_{i,t}$ are defined in \cref{eq:z-and-alpha} in the proof of \cref{lemma:wt}. Furthermore, the following properties are satisfied: (i) $z_1, z_2, \cdots, z_T \iidsim \normal(0, I_k)$; (ii) For all $1 \leq t \leq T$ we have $\sum_{i = 1}^t \alpha_{i, t}^2 = 1$; (iii) $w_t = \sum_{i = 1}^t \alpha_{i, t} z_i$. Then, we can use the same covering argument as in the proof of \cref{lemma:wt} to show that
	\begin{equation*}
		\P \left( \sup_{\alpha \in \S^{T - 1}} \norm{\sum_{i=1}^{T} \alpha_i z_i}_6^6 \ge C k \right) \le C^T \exp (- C k^{1/3}).
	\end{equation*}
    By our assumption, $T \ll k^{1/3}$. Therefore, $\norm{w_t}_6^6 \le C k$ with high probability. As a consequence, we deduce that
    \begin{equation*}
    	\norm{f_t}_2^2 \le C \left( k + T^{9/2} (\log k)^{3/2} \frac{\norm{\Pi_{x_0}^{\perp} \tx_{t-1}}_2^6}{d^3} + \|\Pi_{x_0}^{\perp}(\tx_{t - 1})\|_2^6 \right) \le C \left( k + \|\Pi_{x_0}^{\perp}(\tx_{t - 1})\|_2^6 \right),
    \end{equation*}
    since $T \ll d^{1/2}$. This completes the proof.

\subsection{Proof of Lemma~\ref{lemma:term-I}}
Notice that 
	\begin{align}\label{eq:dcmp}
		\frac{\langle \Pi_{F_{1:i}}^{\perp} \tA_{i + 1} \txp_i, f_t \rangle}{\|\txp_i\|_2} = \frac{\langle  \tA_{i + 1} \txp_i, f_t \rangle}{\|\txp_i\|_2} - \frac{\langle \Pi_{F_{1:i}} \tA_{i + 1} \txp_i, f_t \rangle}{\|\txp_i\|_2}.
	\end{align}
	We first upper bound the second term on the right hand side of \cref{eq:dcmp}. Applying Cauchy-Schwarz inequality implies that
	\begin{align}\label{eq:parallel}
		\frac{\vert \langle \Pi_{F_{1:i}} \tA_{i + 1} \txp_i, f_t \rangle \vert}{\|\txp_i\|_2} \leq \frac{\|\Pi_{F_{1:i}} \tA_{i + 1} \txp_i\|_2}{\|\txp_i\|_2} \cdot \|f_t\|_2.
	\end{align}
	Since $\tA_{i + 1} \perp \cF_{i, i}$, we can deduce that $\sqrt{d}\|\Pi_{F_{1:i}} \tA_{i + 1} \txp_i\|_2 / \|\txp_i\|_2 \overset{d}{=} \sqrt{X_D}$, where $1 \leq D \leq i$ is the rank of $F_{1:i}$ and $X_D$ is a chi-squared random variable with $D$ degrees of freedom. By Bernstein's inequality (\cref{lemma:bernstein}), we obtain that with high probability for all $0 \leq i \leq T$, $\|\Pi_{F_{1:i}} \tA_{i + 1} \txp_i\|_2 / \|\txp_i\|_2 \leq C \sqrt{T / d}$ for some absolute constant $C > 0$. Applying this result and \cref{lemma:bound-ft}, we conclude from \cref{eq:parallel} that, there exists a positive absolute constant $C$, such that with high probability for all $0 \leq i < t \leq T$:
	\begin{align*}
		\frac{\vert \langle \Pi_{F_{1:i}} \tA_{i + 1} \txp_i, f_t \rangle \vert}{\|\txp_i\|_2} \leq C \sqrt{\frac{T}{d}} \cdot \left( \sqrt{k} + \|\Pi_{x_0}^{\perp}(\tx_{t - 1})\|_2^3 \right). 
	\end{align*}
	Next, we consider ${\langle  \tA_{i + 1} \txp_i, f_t \rangle} / {\|\txp_i\|_2}$. Direct computation implies that
	\begin{align*}
		\frac{\langle  \tA_{i + 1} \txp_i, f_t \rangle}{\|\txp_i\|_2} =& \frac{\langle \tA_{i + 1} \txp_i, (w_t - \eta_t + v_t)^3 \rangle}{\|\txp_i\|_2} \\
		= & \frac{\langle \tA_{i + 1} \txp_i, w_t^3 \rangle}{\|\txp_i\|_2}  + \frac{3\langle \tA_{i + 1} \txp_i, w_t^2(v_t - \eta_t)\rangle}{\|\txp_i\|_2} + \frac{3\langle \tA_{i + 1} \txp_i, w_t(v_t - \eta_t)^2\rangle}{\|\txp_i\|_2} + \frac{\langle \tA_{i + 1} \txp_i, (v_t - \eta_t)^3 \rangle}{\|\txp_i\|_2}. 
	\end{align*}
	In what follows, we analyze each of the terms above, separately. Recall that we have defined $z_i$ and $\alpha_{i,t}$ in \cref{eq:z-and-alpha}. Using the representation $w_t = \sum_{i = 1}^t \alpha_{i,t} z_i$, we can then reformulate the first summand above as follow:
	\begin{equation*}
		\frac{\langle \tA_{i + 1} \txp_i, w_t^3 \rangle}{\|\txp_i\|_2} = \frac{1}{\sqrt{d}} \langle z_{i + 1}, ( \mbox{$\sum_{j = 1}^t \alpha_{j,t} z_j$} )^3 \rangle.
	\end{equation*}
	We then show that the above quantity concentrates around its expectation uniformly for $\alpha_t \in \S^{t - 1}$ and $t \in [T]$, via a covering argument similar to that in the proof of \cref{lemma:wt}. First, note that for any fixed $\alpha_t \in \S^{t - 1}$, one has
	\begin{equation*}
		\left( z_{i + 1}, \sum_{j = 1}^t \alpha_{j,t} z_j \right) \stackrel{d}{=} \left( \alpha_{i+1, t} z + \sqrt{1 - \alpha_{i+1, t}^2} g, z \right),
	\end{equation*}
    where $z, g \sim \sN (0, I_k)$ are mutually independent. This further implies that
    \begin{equation*}
    	\E \left[ \frac{\langle \tA_{i + 1} \txp_i, w_t^3 \rangle}{\|\txp_i\|_2} \right] = \frac{1}{\sqrt{d}} \E \left[ \langle \alpha_{i+1, t} z, z^3 \rangle \right] = \frac{3 k}{\sqrt{d}} \alpha_{i+1, t} = \frac{3k}{{d}} \cdot \frac{\langle \txp_i, \tx_{t - 1} \rangle}{\|\txp_i\|_2}.
    \end{equation*}
	Moreover, using \cref{lemma:conc_non_exp}, we deduce that there exist constants $C_0, C_1, C_2 > 0$, such that
	\begin{align}\label{eq:1}
		\P \left( \left| \frac{1}{\sqrt{d}} \langle z_{i + 1}, ( \mbox{$\sum_{j = 1}^t \alpha_{j,t} z_j$} )^3 \rangle - \frac{3 k}{\sqrt{d}} \alpha_{i+1, t} \right| \ge C_0 \sqrt{\frac{k}{d}} \cdot \sqrt{T \log k} \right) \le C_1 \exp (- C_2 T \log k),
	\end{align}
	where $C_1$ and $C_2$ depend on $C_0$, and $C_2 \to \infty$ as $C_0 \to \infty$. Let $\veps > 0$ be a small constant (to be determined later), for $\alpha_t, \alpha_t' \in \S^{t - 1}$ satisfying $\norm{\alpha_t - \alpha_t'}_2 \le \veps$, we have
	\begin{align*}
		& \left\vert \left( \frac{1}{\sqrt{d}} \langle z_{i + 1}, ( \mbox{$\sum_{j = 1}^t \alpha_{j,t} z_j$} )^3 \rangle - \frac{3 k}{\sqrt{d}} \alpha_{i+1, t} \right) - \left( \frac{1}{\sqrt{d}} \langle z_{i + 1}, ( \mbox{$\sum_{j = 1}^t \alpha_{j,t}' z_j$} )^3 \rangle - \frac{3 k}{\sqrt{d}} \alpha_{i+1, t}' \right) \right\vert \\
		\le\, & \frac{3k \veps}{\sqrt{d}} + \frac{1}{\sqrt{d}} \left\vert \left\langle z_{i + 1}, ( \mbox{$\sum_{j = 1}^t \alpha_{j,t} z_j$} )^3 - ( \mbox{$\sum_{j = 1}^t \alpha_{j,t}' z_j$} )^3 \right\rangle \right\vert \\
		\le\, & \frac{3k \veps}{\sqrt{d}} + \frac{1}{\sqrt{d}} \norm{z_{i+1}}_{\infty} \cdot \norm{( \mbox{$\sum_{j = 1}^t \alpha_{j,t} z_j$} )^3 - ( \mbox{$\sum_{j = 1}^t \alpha_{j,t}' z_j$} )^3}_1 \\
		\le\, & \frac{3k \veps}{\sqrt{d}} + \frac{1}{\sqrt{d}} \norm{z_{i+1}}_{\infty} \cdot \frac{3}{2} \left( \norm{\mbox{$\sum_{j = 1}^t \alpha_{j,t} z_j$}}_2^2 + \norm{\mbox{$\sum_{j = 1}^t \alpha_{j,t}' z_j$}}_2^2 \right) \cdot \norm{\mbox{$\sum_{j = 1}^t \alpha_{j,t} z_j$} - \mbox{$\sum_{j = 1}^t \alpha_{j,t}' z_j$}}_{\infty} \\
		\le\, & \frac{3k \veps}{\sqrt{d}} + \frac{3 \sqrt{t} \veps}{\sqrt{d}} \sup_{1 \le j \le t} \norm{z_{j}}_{\infty}^2 \cdot \sup_{\alpha_t \in \S^{t-1}} \norm{\mbox{$\sum_{j = 1}^t \alpha_{j,t} z_j$}}_2^2,
	\end{align*} 
    where the last line follows from Cauchy-Schwarz inequality. According to \cref{lemma:gaussian-tail}, we know that there exists a numerical constant $C > 0$, such that with probability $1 - o_d(1)$, we have $\|z_i\|_{\infty} \leq C\sqrt{\log k}$ for all $i \in [T]$. Moreover, using a covering argument similar to the proof of Lemma~\ref{lemma:wt}, we deduce that with high probability,
    \begin{equation*}
    	\sup_{\alpha_t \in \S^{t-1}} \norm{\mbox{$\sum_{j = 1}^t \alpha_{j,t} z_j$}}_2^2 \le C k \ \text{for all} \ t \in [T],
    \end{equation*}
    thus leading to the following estimate:
    \begin{equation*}
    	\left\vert \left( \frac{1}{\sqrt{d}} \langle z_{i + 1}, ( \mbox{$\sum_{j = 1}^t \alpha_{j,t} z_j$} )^3 \rangle - \frac{3 k}{\sqrt{d}} \alpha_{i+1, t} \right) - \left( \frac{1}{\sqrt{d}} \langle z_{i + 1}, ( \mbox{$\sum_{j = 1}^t \alpha_{j,t}' z_j$} )^3 \rangle - \frac{3 k}{\sqrt{d}} \alpha_{i+1, t}' \right) \right\vert \le C k \veps \sqrt{\frac{T}{d}} \log k.
    \end{equation*}
	Therefore, we can apply an $\veps$-net covering argument with $\veps = 1/k$ on $\S^{t-1}$, and choose $C_0$ to be large enough so that $C_2 > 1$. This finally implies that with high probability, for all $t \in [T]$ we have
	\begin{equation*}
		\sup_{\alpha_t \in \S^{t - 1}} \left| \frac{1}{\sqrt{d}} \langle z_{i + 1}, ( \mbox{$\sum_{j = 1}^t \alpha_{j,t} z_j$} )^3 \rangle - \frac{3 k}{\sqrt{d}} \alpha_{i+1, t} \right| \le C \sqrt{\frac{k}{d}} \cdot \sqrt{T \log k},
	\end{equation*}
    which further implies that
    \begin{equation*}
    	\left\vert \frac{\langle \tA_{i + 1} \txp_i, w_t^3 \rangle}{\|\txp_i\|_2} - \frac{3k}{{d}} \cdot \frac{\langle \txp_i, \tx_{t - 1} \rangle}{\|\txp_i\|_2} \right\vert \le C \sqrt{\frac{k}{d}} \cdot \sqrt{T \log k}.
    \end{equation*}
	
	Now we try to upper bound the remainders. We already know that there exists a numerical constant $C > 0$, such that with probability $1 - o_d(1)$, we have $\|z_i\|_{\infty} \leq C\sqrt{\log k}$ for all $i \in [T]$. Therefore, with probability $1 - o_d(1)$, the following holds for all $t \in [T]$:
	\begin{align*}
		\|w_t\|_{\infty} \leq\, & \sum_{i = 1}^t |\alpha_{i,t}| \cdot \|z_i\|_{\infty} \leq C \sqrt{\log k} \left( \vert \alpha_{1, t} \vert + \sum_{i=2}^{t} \vert \alpha_{i, t} \vert \right) \le C \sqrt{\log k} \left(\vert \alpha_{1, t} \vert + \sqrt{t \sum_{i=2}^{t} \alpha_{i, t}^2 } \right) \\
		=\, & C \sqrt{\log k} \left( \frac{1}{\sqrt{d}} \|\Pi_{x_0} (\tx_{t - 1})\|_2 + \sqrt{\frac{t}{d}} \|\Pi_{x_0}^{\perp}(\tx_{t - 1})\|_2 \right) \le C \sqrt{\log k} \left( 1 + \sqrt{\frac{T}{d}} \|\Pi_{x_0}^{\perp}(\tx_{t - 1})\|_2 \right).
	\end{align*}
	According to power mean inequality, there exists a numerical constant $C > 0$, such that with probability $1 - o_d(1)$, for all $1 \leq i+1 \leq t \leq T$ we have
	\begin{align}\label{eq:2}
	\begin{split}
		\frac{3|\langle \tA_{i + 1} \txp_i, w_t(v_t - \eta_t)^2\rangle|}{\|\txp_i\|_2} \leq & \frac{3}{\sqrt{d}} \|z_{i + 1}\|_{\infty}\cdot \|w_t\|_{\infty} \cdot \|v_t - \eta_t\|_2^2 \\
		 \leq & \frac{C \log k}{\sqrt{d}} \cdot \left( 1 + \sqrt{\frac{T}{d}} \|\Pi_{x_0}^{\perp}(\tx_{t - 1})\|_2 \right) \big(\|\eta_t\|_2^2 + \|v_t\|_2^2\big) \\
		\overset{(i)}{\leq} & \frac{C \log k}{\sqrt{d}} \cdot \left( 1 + \sqrt{\frac{T}{d}} \|\Pi_{x_0}^{\perp}(\tx_{t - 1})\|_2 \right) \|\Pi_{x_0}^{\perp}(\tx_{t - 1})\|_2^2.
	\end{split}
	\end{align}
	In the above equation, \emph{(i)} follows from \cref{lemma:etat} and \cref{lemma:norm-of-vt}. Similarly, we can conclude that there exists a numerical constant $C > 0$, such that with probability $1 - o_d(1)$, for all $1 \leq i + 1 \leq t \leq T$ the following holds:
	\begin{align}\label{eq:3}
	\begin{split}
		\frac{3|\langle \tA_{i + 1} \txp_i, w_t^2(v_t - \eta_t)\rangle|}{\|\txp_i\|_2} \leq & \frac{3}{\sqrt{d}}\|z_{i + 1}\|_{\infty} \cdot \|w_t^2 (v_t - \eta_t) \|_1 \\
		\overset{(ii)}{\leq} & \frac{C \sqrt{\log k}}{\sqrt{d}} \norm{w_t^2}_2 \cdot \|v_t - \eta_t\|_2 \\
		\leq & \frac{C \sqrt{\log k}}{\sqrt{d}} \norm{w_t}_4^2 \cdot \big(\|\eta_t \|_2 + \|v_t\|_2 \big) \\
		\overset{(iii)}{\leq} & \frac{C \sqrt{k \log k}}{\sqrt{d}} \cdot \norm{\Pi_{x_0}^{\perp}(\tx_{t - 1})}_2.
	\end{split}
	\end{align}
	In the above inequalities, \emph{(ii)} is due to \Holder's inequality and $\norm{z_{i+1}}_{\infty} \le C \sqrt{\log k}$, and \emph{(iii)} is due to \cref{lemma:wt}, \cref{lemma:etat}, and \cref{lemma:norm-of-vt}. Finally, according to the power mean inequality, we obtain that with probability $1 - o_d(1)$,
	\begin{align}\label{eq:4}
	\begin{split}
		\frac{|\langle \tA_{i + 1} \txp_i, (v_t - \eta_t)^3 \rangle |}{\|\txp_i\|_2} \leq\, & \frac{4}{\sqrt{d}}\|z_{i + 1}\|_{\infty} \cdot \left( \|v_t\|_3^3 + \|\eta_t\|_3^3 \right) \\
		\leq\, & C\sqrt{\frac{ {\log k}}{{d}}} \cdot \big( \|v_t\|_2^3 + \|\eta_t\|_2^3 \big) \\
		\leq\, & 2C\sqrt{\frac{ {\log k}}{{d}}} \cdot \norm{\Pi_{x_0}^{\perp}(\tx_{t - 1})}_2^3.
	\end{split}
	\end{align}
	\cref{lemma:term-I} then follows from \cref{eq:parallel,eq:1,eq:2,eq:3,eq:4} and our assumptions.

\section{Increasing objective function: Proof of Theorem~\ref{thm:increasing}}\label{sec:pf_increasing}

This section will be devoted to proving \cref{thm:increasing}. The main technical lemma employed to prove the theorem can be stated as follows:
\begin{lem}\label{lemma:increasing}
	Under the condition of \cref{thm:increasing}, for any $t \in [T_c]$ we have 
	\begin{align}
		& y_t = \sum_{i = 1}^t \alpha_{i,t}z_i + \sum_{j = 1}^{t - 1} \zeta_{j,t} f_j + \veps_t, \label{eq:increasing1} \\
		& \frac{d}{9k^2}\|x_t\|_2^2 = \frac{\langle x_0, x_t \rangle^2}{9k^2} + o_P(1) = 1 + o_P(1), \label{eq:increasing2}
	\end{align}
	where as in the proof of \cref{lemma:wt} we have
	\begin{align*}
		z_i = \tA_i \txp_{i - 1} \cdot \frac{\sqrt{d}}{\|\txp_{i - 1}\|_2}, \qquad \alpha_{i, t} = \frac{\langle \txp_{i - 1}, \tx_{t - 1} \rangle}{\sqrt{d} \|\txp_{i - 1}\|_2}.
	\end{align*}
	Furthermore, the above quantities satisfy:
\begin{align}
    & \alpha_{1,t} = 1 + o_P(1), \label{eq:increasing4} \\
    & \frac{3k}{d}\zeta_{j,t}= 1 + o_P(1),  \label{eq:increasing5}\\
    & \|\veps_t\|_2 = o_P(d / \sqrt{k}). \label{eq:increasing6}
\end{align}
In addition, with probability $1 - o_d(1)$, for all $t \in [T_c]$ we have
	\begin{align}\label{eq:increasing3}
		\frac{1}{k}\|f_t\|_2^2 \leq 20. 
	\end{align}
\end{lem}
We then proceed to prove \cref{thm:increasing}. For the sake of simplicity, define
\begin{equation*}
    w_t = \sum_{i = 1}^t \alpha_{i,t}z_i, \quad \nu_t = \sum_{j = 1}^{t - 1} \zeta_{j, t} f_j + \veps_t.
\end{equation*}
Then, \cref{lemma:increasing} implies that
\begin{align*}
    \|\nu_t\|_2 \le \sum_{j=1}^{t-1} \vert \zeta_{j, t} \vert \cdot \norm{f_j}_2 + \norm{\veps_t}_2 \le C \cdot \sum_{j=1}^{t-1} O_P \left( \frac{d}{\sqrt{k}} \right) + o_P \left( \frac{d}{\sqrt{k}} \right) = O_P \left( \frac{d}{\sqrt{k}} \right),
\end{align*}
i.e., $\|\nu_t\|_2^2 = O_P(d^2 / k)$. Recall that in the proof of Lemma~\ref{lemma:wt}, we have shown that $\sum_{i=1}^{t} \alpha_{i, t}^2 = 1$ and $z_i \iidsim \normal(0, I_k)$. Since $\alpha_{1, t} = 1 + o_P(1)$ and $T_c$ is a constant, we can use standard concentration arguments for Gaussian random variables to deduce that
\begin{equation*}
    \|w_t^3\|_2^2 = \norm{w_t}_6^6 = 15k + O_P(\sqrt{k}), \quad \|w_t\|_{\infty} = O_P(\log k).
\end{equation*}
Using \cref{eq:increasing1} from \cref{lemma:increasing}, we see that 
\begin{align*}
    \mathcal{S}(\tx_{t - 1}) & = \norm{y_t}_4^4 =  \langle w_t^4, 1 \rangle + 4 \langle w_t^3, \nu_t \rangle + 6 \langle w_t^2, \nu_t^2 \rangle + 4 \langle w_t, \nu_t^3 \rangle + \langle \nu_t^4, 1 \rangle \\
    & = \langle w_t^4, 1 \rangle + 4 \sum_{j = 1}^{t - 1} \zeta_{j, t} \langle w_t^3, f_j \rangle + 4 \langle w_t^3, \veps_t \rangle + 6 \langle w_t^2, \nu_t^2 \rangle + 4 \langle w_t, \nu_t^3 \rangle + \langle \nu_t^4, 1 \rangle. 
\end{align*}
Next, we analyze the terms above, respectively. Similar to the previous argument, we can show that $\langle w_t^4, 1 \rangle = \norm{w_t}_4^4 = 3k + O_P(\sqrt{k})$. Leveraging \cref{eq:increasing4,eq:increasing5}, we obtain that
\begin{align*}
    4 \sum_{j = 1}^{t - 1} \zeta_{j, t} \langle w_t^3, f_j \rangle = 4 \sum_{j = 1}^{t - 1} \zeta_{j, t} \langle w_t^3, (y_j)^3 \rangle = 4 \sum_{j = 1}^{t - 1} \zeta_{j, t} \langle w_t^3, (w_j + \nu_j)^3 \rangle.
\end{align*}
For any $1 \le j \le t-1$, using again standard concentration arguments, we obtain that $\left\langle w_t^3, (w_j + \nu_j)^3 \right\rangle = 15 k + o_P(k)$. Note that $\zeta_{j, t} = d / ((3 + o_P(1)) k)$, it follows that
\begin{equation*}
    4 \sum_{j = 1}^{t - 1} \zeta_{j, t} \langle w_t^3, f_j \rangle = 4 \sum_{j=1}^{t-1} (5 + o_P(1)) d = 20(t-1) d + o_P(d).
\end{equation*}
Applying Cauchy–Schwarz inequality we see that
\begin{equation*}
    4|\langle w_t^3, \veps_t \rangle| \leq 4\|w_t^3\|_2 \cdot \|\veps_t\|_2 = O_P (\sqrt{k}) \cdot o_P \left( \frac{d}{\sqrt{k}} \right) = o_P(d).
\end{equation*}
Furthermore, the following results hold: 
\begin{align*}
    6|\langle w_t^2, \nu_t^2 \rangle| & \leq 6 \|w_t\|_{\infty}^2 \cdot \|\nu_t\|_2^2 = O_P(d^2 (\log k)^2 / k), \\
    4 |\langle w_t, \nu_t^3 \rangle| & \leq 4\|w_t\|_{\infty} \cdot \|\nu_t\|_3^3 \leq 4\|w_t\|_{\infty} \cdot \|\nu_t\|_2^3 = O_P(d^3 \log k / k^{3/2}), \\
    |\langle \nu_t^4, 1 \rangle| & = \norm{\nu_t}_4^4 \leq \|\nu_t\|_2^4 = O_P(d^4 / k^2).
\end{align*}
Combining these estimates and using the assumption that $d^{3/2} \ll k \ll d^2$, we conclude that $S(\tx_{t - 1}) = 3k + 20(t - 1)d + o_P(d)$, thus completing the proof of the theorem. 

\subsection{Proof of \cref{lemma:increasing}}

We prove the lemma via induction over $t$. 
\vspace{-0.3cm}
\subsubsection*{Base case}
For the base case $t = 1$, from \cref{lemma:conditioning} we immediately see that $y_1 = \alpha_{1,1} z_1$, thus proves \cref{eq:increasing1}.   Furthermore, by definition $\alpha_{1,1} = 1$, which justifies \cref{eq:increasing4} for the base case. \cref{eq:increasing3} follows immediately from the law of large numbers. Again by \cref{lemma:conditioning}, we have 
	\begin{align}\label{eq:decomp-x1}
		x_1 = \tx_0 \cdot \frac{\langle y_1, f_1 \rangle}{d} + \Pi_{x_0}^{\perp} \bar{A}_1^{\top} f_1. 
	\end{align}
	Using the law of large numbers, we obtain that $\langle y_1, f_1\rangle / d = 3k / d + o_P(k / d)$, $\|\Pi_{x_0}^{\perp}\bA_1^{\top} f_1\|_2^2 = \|f_1\|_2^2 + o_P(k) = 15k + o_P(k)$. We then discover that \cref{eq:increasing2} for the base case follows, since by \cref{eq:decomp-x1} we have $\|x_1\|_2^2 = \langle y_1, f_1\rangle^2 / d + \|\Pi_{x_0}^{\perp}\bA_1^{\top} f_1\|_2^2$ and $\langle x_0, x_1 \rangle^2 / 9k^2 = \langle y_1, f_1 \rangle^2 / 9k^2$. This completes the proof of \cref{eq:increasing2} for the base case. We note that \cref{eq:increasing5} does not apply for the base case. 

\vspace{-0.3cm}
\subsubsection*{Proof of Eq.~\eqref{eq:increasing1}, \eqref{eq:increasing4}, \eqref{eq:increasing5}, \eqref{eq:increasing6} for $t = s + 1$}
	
Suppose the claims hold for $t = s$. Next, we prove that they also hold for $t = s + 1$ via induction. Leveraging \cref{eq:cond-yt} in \cref{lemma:conditioning}, we obtain that 
	\begin{align}\label{eq:44}
		y_{s + 1} = \sum_{i = 1}^{s + 1} \alpha_{i,s + 1} z_i  + \sum_{i = 2}^{s + 1} \eta_{i, s + 1} \fp_{i - 1}, 
	\end{align}
	where for $2 \leq i \leq s + 1$
	\begin{align*}
		\eta_{i, s + 1} = \underbrace{\frac{\langle x_{i - 1}, \txp_{i - 1} \rangle}{\|\fp_{i - 1}\|_2^2} \cdot \frac{\langle \txp_{i - 1}, \tx_s \rangle}{\| \txp_{i - 1} \|_2^2}}_{a_{i,s + 1}} - \underbrace{\sum_{j = i}^{s  + 1} \frac{\langle \fp_{i - 1}, \tA_j \txp_{j - 1} \rangle}{\|\fp_{i - 1}\|_2^2} \cdot \frac{\langle \txp_{j - 1}, \tx_s \rangle}{\|\txp_{j - 1}\|_2^2}}_{b_{i, s + 1}}.
	\end{align*}
	We then proceed to prove that $\|b_{i,s + 1} \fp_{i - 1}\|_2 = O_P(1)$. To this end, it suffices to show for every $j \in \{i, i + 1, \cdots, s + 1\}$,
	\begin{align}\label{eq:45}
		\frac{|\langle \fp_{i - 1}, \tA_j \txp_{j - 1} \rangle|}{\|\fp_{i - 1}\|_2^2} \cdot \frac{|\langle \txp_{j - 1}, \tx_s \rangle|}{\|\txp_{j - 1}\|_2^2} \cdot \|\fp_{i - 1}\|_2 = O_P(1).
	\end{align}
	Note that for $j \geq i$ we have $\tA_j \perp \cF_{j -1,j-1}$, thus 
	\begin{align}\label{eq:46}
		\frac{\langle \fp_{i - 1}, \tA_j \txp_{j - 1} \rangle}{\|\fp_{i - 1}\|_2 \| \txp_{j - 1} \|_2} \overset{d}{=} \normal(0, d^{-1}). 
	\end{align}
	Applying Cauchy-Schwarz inequality, we see that
	\begin{align}\label{eq:47}
		\frac{|\langle \txp_{j - 1}, \tx_s \rangle|}{\|\txp_{j - 1}\|_2} \leq \|\tx_s\|_2 = \sqrt{d}. 
	\end{align}
	Combining \cref{eq:46,eq:47}, we deduce that \cref{eq:45} holds for all $j \in \{i, i + 1, \cdots, s + 1\}$. Then we switch to consider $a_{i, s + 1}$. Using \cref{eq:cond-xt} and \cref{eq:ht+1}, we have the following decomposition:
	\begin{align*}
		a_{i, s + 1} = & \frac{\langle x_{i - 1}, \txp_{i - 1} \rangle}{\|\fp_{i - 1}\|_2^2} \cdot \frac{\langle h_i, f_s \rangle}{\| \txp_{i - 1} \|_2^2} \cdot \frac{\sqrt{d}}{\|x_s\|_2} \\
		= & \underbrace{\frac{\langle x_{i - 1}, \txp_{i - 1} \rangle}{\|\fp_{i - 1}\|_2^2} \cdot \frac{\langle \fp_{i - 1}, f_s \rangle}{\| \txp_{i - 1} \|_2^2} \cdot \frac{\sqrt{d}}{\|x_s\|_2} \cdot \frac{\langle x_{i - 1}, \txp_{i - 1} \rangle}{\|\fp_{i - 1}\|_2^2}}_{p_{i, s + 1}} + \underbrace{\frac{\langle x_{i - 1}, \txp_{i - 1} \rangle}{\|\fp_{i - 1}\|_2^2} \cdot \frac{\langle \Pi_{F_{1:i-1}}^{\perp} \tA_i \txp_{i - 1}, f_s \rangle}{\| \txp_{i - 1} \|_2^2} \cdot \frac{\sqrt{d}}{\|x_s\|_2}}_{q_{i, s + 1}}.  
	\end{align*}
	We then analyze $p_{i, s + 1}$ and $q_{i, s + 1}$, respectively. Combining \cref{eq:cond-xt}, the law of large numbers and the fact that $\bA_t \perp \cF_{t - 1, t}$, we see that for all $t \in \NN_+$
	\begin{align}\label{eq:47.5}
		\frac{\langle x_t, \xp_t \rangle}{\|\fp_t\|_2^2} = 1 + O_P(d^{-1/2}).
	\end{align}
	Combining the above analysis with \cref{eq:increasing2} from previous induction steps, we conclude that
	\begin{align}\label{eq:pis+1}
		p_{i, s + 1} = \left(1 + o_P(1)) \right) \times \frac{\langle \fp_{i - 1}, f_s \rangle}{\|\fp_{i - 1}\|_2^2} \times \frac{d}{3k}.
	\end{align}
	Next, we consider $q_{i, s + 1}$. We further decompose $q_{i, s + 1}$ as the difference of the following two terms:
	\begin{align*}
		q_{i, s + 1} = \underbrace{\frac{\langle x_{i - 1}, \txp_{i - 1} \rangle}{\|\fp_{i - 1}\|_2^2} \cdot \frac{\langle  \tA_i \txp_{i - 1}, f_s \rangle}{\| \txp_{i - 1} \|_2^2} \cdot \frac{\sqrt{d}}{\|x_s\|_2}}_{c_{i,s + 1}} - \underbrace{\sum_{j = 1}^{i - 1}\frac{\langle x_{i - 1}, \txp_{i - 1} \rangle}{\|\fp_{i - 1}\|_2^2} \cdot \frac{\langle \fp_j, \tA_i \txp_{i - 1} \rangle}{\|\fp_j\|_2^2} \cdot \frac{\langle \fp_j, f_s \rangle}{\| \txp_{i - 1} \|_2^2} \cdot \frac{\sqrt{d}}{\|x_s\|_2}}_{e_{i,s + 1}}.
	\end{align*} 
	Note that for all $1 \leq j \leq i - 1 \leq s$, leveraging \cref{eq:46}, \eqref{eq:47.5} and \cref{eq:increasing2} from previous induction steps, we  obtain that
	\begin{align*}
		& \|\fp_{i - 1}\|_2 \cdot\frac{|\langle x_{i - 1}, \txp_{i - 1} \rangle|}{\|\fp_{i - 1}\|_2^2} \cdot \frac{|\langle \fp_j, \tA_i \txp_{i - 1} \rangle |}{\|\fp_j\|_2^2} \cdot \frac{|\langle \fp_j, f_s \rangle|}{\| \txp_{i - 1} \|_2^2} \cdot \frac{\sqrt{d}}{\|x_s\|_2} \\
		= & \frac{|\langle \fp_j, f_s \rangle|}{\|\fp_j\|_2} \cdot \frac{\sqrt{d}}{\|x_s\|_2} \cdot \frac{|\langle \fp_j, \tA_i \txp_{i - 1} \rangle |}{\|\fp_j\|_2 \|\txp_{i - 1}\|_2} \cdot \left(1 + o_P(1) \right) \\
		\leq &\frac{\sqrt{d}}{k} \|f_s\|_2 \cdot O_P(1).
	\end{align*}
	Leveraging \cref{eq:increasing3} from previous induction steps, we obtain that with probability $1 - o_d(1)$, the above quantity is no larger than $O_P(\sqrt{d / k} )  $. This further implies that $\|e_{i,s + 1} \fp_{i - 1}\|_2 =O_P(\sqrt{d / k} ) $. 
 
 Finally, we analyze $c_{i, s + 1}$. By \cref{eq:44} from previous induction steps, we see that 
	\begin{align}\label{eq:51}
	\begin{split}
		\frac{\langle  \tA_i \txp_{i - 1}, f_s \rangle}{\| \txp_{i - 1} \|_2} = & \frac{1}{\sqrt{d}} \Big \langle z_i, \big(\sum_{j = 1}^s \alpha_{j,s}z_j + \sum_{j = 1}^{s - 1} \zeta_{j,s} f_j + \veps_s \big)^3 \Big \rangle \\
		= & \frac{1}{\sqrt{d}} \Big \langle z_i, \big(\sum_{j = 1}^s \alpha_{j,s}z_j  \big)^3 \Big \rangle + \frac{3}{\sqrt{d}} \Big \langle z_i(\sum_{j = 1}^s \alpha_{j,s}z_j)^2, \sum_{j = 1}^{s - 1} \zeta_{j,s} f_j + \veps_s \Big \rangle + \\
  &\frac{3}{\sqrt{d}} \Big \langle z_i(\sum_{j = 1}^s \alpha_{j,s}z_j), \big(\sum_{j = 1}^{s - 1} \zeta_{j,s} f_j + \veps_s \big)^2\Big \rangle  + \frac{1}{\sqrt{d}} \Big \langle z_i, \big(\sum_{j = 1}^{s - 1} \zeta_{j,s} f_j + \veps_s \big)^3 \Big \rangle. 
	\end{split}
	\end{align}
	Below we analyze terms in \cref{eq:51}, separately. 
	\begin{align}
		 \frac{1}{\sqrt{d}} \Big \langle z_i, \big(\sum_{j = 1}^s \alpha_{j,s}z_j  \big)^3 \Big \rangle =& \frac{3k\langle \txp_{i - 1}, \tx_{s - 1} \rangle}{{d} \|\txp_{i - 1}\|_2} + O_P\Big( \sqrt{\frac{k}{d}} \Big), \label{eq:52} \\
		  \frac{3}{\sqrt{d}} \Big| \Big \langle z_i(\sum_{j = 1}^s \alpha_{j,s}z_j)^2, \sum_{j = 1}^{s - 1} \zeta_{j,s} f_j + \veps_s \Big \rangle \Big| \leq & \frac{3}{\sqrt{d}}\big\|z_i^2(\sum_{j = 1}^s \alpha_{j,s}z_j)^2 \big\|_2 \times \left( \sum_{j = 1}^{s - 1} |\zeta_{j,s}| \cdot \|f_j\|_2 + \|\veps_s\|_2 \right) \nonumber \\
  =&  O_P(\sqrt{d}), \label{eq:53}\\
		 \frac{3}{\sqrt{d}} \Big| \Big \langle z_i(\sum_{j = 1}^s \alpha_{j,s}z_j), \big(\sum_{j = 1}^{s - 1} \zeta_{j,s} f_j + \veps_s \big)^2\Big \rangle \Big| \leq & \frac{3 T_c \|z_i\|_{\infty} \max_{1 \leq j \leq T_c}\|z_j\|_{\infty} }{\sqrt{d}} \cdot \Big\| \big(\sum_{j = 1}^{s - 1} \zeta_{j,s} f_j + \veps_s \big)^2 \Big\|_1  \nonumber \\
   = &  \frac{3 T_c \|z_i\|_{\infty} \max_{1 \leq j \leq T_c}\|z_j\|_{\infty} }{\sqrt{d}} \cdot \Big\| \sum_{j = 1}^{s - 1} \zeta_{j,s} f_j + \veps_s  \Big\|_2^2 \nonumber \\
		  \leq & \frac{3 T_c^2 \|z_i\|_{\infty} \max_{1 \leq j \leq T_c}\|z_j\|_{\infty} }{\sqrt{d}} \cdot \Big( \sum_{j = 1}^{s - 1} \zeta_{j,s}^2 \|f_j\|_2^2 + \|\veps_s\|_2^2  \Big) \nonumber \\
		 \qquad \qquad \qquad \qquad \qquad \,\,\,\,\,\,\,\, \, \overset{w.h.p.}{\leq} & \frac{100 T_c^3 (\log k)^2}{\sqrt{d}} \times \frac{d^2}{k} = O_P(d^{1.5001} / k).  \label{eq:54}
	\end{align}
\cref{eq:52} is by the law of large numbers. In \cref{eq:53}, we employ \cref{eq:increasing5,eq:increasing6,eq:increasing3} from previous induction steps.  \cref{eq:54} is by \cref{eq:increasing5,eq:increasing6,eq:increasing3} from  induction  and the fact that with high probability, $\|z_i\|_{\infty} \leq \log k$ for all $i \in [T_c]$.

	Combining \cref{eq:51,eq:52,eq:53,eq:54} and \cref{eq:increasing2} from induction, we obtain that 
	\begin{align}\label{eq:cis+1}
		\Big\|c_{i,s + 1}\fp_{i - 1} - \frac{\langle \txp_{i - 1}, \tx_{s - 1} \rangle}{\|\txp_{i - 1}\|_2 \|\fp_{i - 1}\|_2}\fp_{i - 1}\Big\|_2 \leq O_P(d^{3/2}/k). 
	\end{align}
	Plugging the definitions of $\{a_{i, s + 1}, b_{i, s + 1}, p_{i, s + 1}, q_{i, s + 1}, c_{i, s + 1}, e_{i, s + 1}\}$ into \cref{eq:44}, we have
	\begin{align*}
		y_{s + 1} - \sum_{i = 1}^{s + 1} \alpha_{i, s + 1} z_i = & \sum_{i = 2}^{s + 1} (a_{i, s + 1} - b_{i, s + 1}) \fp_{i - 1} \\
		= & \sum_{i = 2}^{s + 1}(p_{i, s + 1} + q_{i, s + 1} - b_{i, s + 1} )\fp_{i - 1} \\
		= & \sum_{i = 2}^{s + 1}(p_{i, s + 1} + c_{i, s + 1} - e_{i, s + 1} - b_{i, s + 1} )\fp_{i - 1}.
	\end{align*}
	Recall that we have proved $\|b_{i,s + 1} \fp_{i - 1}\|_2 = O_P(1)$ and  $\|e_{i,s + 1} \fp_{i - 1}\|_2 \leq O_P(\sqrt{d / k})$. Using these results, together with  \cref{eq:pis+1}, \eqref{eq:cis+1} and \cref{eq:increasing3} with $t = s$, we have
%
	%
	\begin{align}\label{eq:57}
		\Big\|y_{s + 1} - \sum_{i = 1}^{s + 1} \alpha_{i, s + 1} z_i - \sum_{i = 2}^{s + 1} \frac{\langle \txp_{i - 1}, \tx_{s - 1} \rangle}{\|\txp_{i - 1}\|_2 \|\fp_{i - 1}\|_2}\fp_{i - 1} - \sum_{i = 2}^{s + 1} \frac{d\langle \fp_{i - 1}, f_s \rangle}{3k\|\fp_{i - 1}\|_2^2} \fp_{i - 1} \Big\|_2 = o_P(d / \sqrt{k}).  
	\end{align}
	Notice that 
 \begin{align*}
     \sum_{i = 2}^{s + 1} \frac{d\langle \fp_{i - 1}, f_s \rangle}{3k\|\fp_{i - 1}\|_2^2} \fp_{i - 1} = \frac{d}{3k} f_s.
 \end{align*}
 Using triangle inequality, we see that 
	\begin{align}\label{eq:58}
		& \Big\|\sum_{i = 2}^{s + 1} \left(  \frac{\langle \txp_{i - 1}, \tx_{s - 1} \rangle}{\|\txp_{i - 1}\|_2 \|\fp_{i - 1}\|_2} - \frac{\langle x_{i - 1}, \txp_{i - 1} \rangle}{\|\fp_{i - 1}\|_2^2} \cdot \frac{\langle \txp_{i - 1}, \tx_{s - 1} \rangle}{\| \txp_{i - 1} \|_2^2} \right)\fp_{i - 1} \Big\|_2  \nonumber \\
		 \leq  & \sum_{i = 2}^{s + 1} \Big| \frac{\langle \txp_{i - 1}, \tx_{s - 1} \rangle}{\|\txp_{i - 1}\|_2 } \cdot \left(1 - \frac{\|\xp_{i - 1}\|_2}{\|\fp_{i - 1}\|_2} \right) \Big|,
	\end{align}
	which by \cref{eq:47.5} is $O_P(1)$. \cref{eq:57,eq:58} and the condition $d \ll k \ll d^2$ together imply that 
	\begin{align}\label{eq:59}
		\Big\|y_{s + 1} - \sum_{i = 1}^{s + 1} \alpha_{i, s + 1} z_i - \sum_{i = 2}^{s + 1} \frac{\langle x_{i - 1}, \txp_{i - 1} \rangle}{\|\fp_{i - 1}\|_2^2} \cdot \frac{\langle \txp_{i - 1}, \tx_{s - 1} \rangle}{\| \txp_{i - 1} \|_2^2}\fp_{i - 1} - \frac{d}{3k} f_s \Big\|_2 = o_P(d / \sqrt{k}). 
	\end{align}
	Using the definitions of $a_{i,s}, b_{i,s}$, we obtain that 
	\begin{align*}
		\sum_{i = 2}^{s + 1} \frac{\langle x_{i - 1}, \txp_{i - 1} \rangle}{\|\fp_{i - 1}\|_2^2} \cdot \frac{\langle \txp_{i - 1}, \tx_{s - 1} \rangle}{\| \txp_{i - 1} \|_2^2}\fp_{i - 1} = & \sum_{i = 2}^s a_{i,s} \fp_{i - 1} \\
		= & y_{s} - \sum_{i = 1}^{s} \alpha_{i,s} z_i + \sum_{i = 2}^s b_{i,s} \fp_{i - 1} \\
		= & \sum_{j = 1}^{s - 1} \zeta_{j,s} f_j + \veps_s + \sum_{i = 2}^s b_{i,s} \fp_{i - 1}.
	\end{align*}
	Since we have proved $\|\sum_{i = 2}^s b_{i,s} \fp_{i - 1} \|_2 = O_P(1)$ and by induction $\|\veps_s\|_2 = o_P(d / \sqrt{k})$, we can set 
	\begin{align*}
		& \veps_{s + 1} = y_{s + 1} - \sum_{i = 1}^{s + 1} \alpha_{i, s + 1} z_i - \sum_{j = 1}^{s - 1} \zeta_{j,s} f_j - \frac{d}{3k} f_s, \\
		& \zeta_{j, s + 1} = \zeta_{j, s} \,\,\,\mbox{ for }j = 1,2, \cdots, s - 1, \\
		& \zeta_{s, s + 1} = \frac{d}{3k}. 
	\end{align*}
        Combining the above analysis, we find that $\|\veps_{s + 1}\|_2 = o_P(d / \sqrt{k})$, thus proves \cref{eq:increasing6}. \cref{eq:increasing5} for $t = s + 1$ is a direct consequence of our induction hypothesis. Thus, we have completed the proof of \cref{eq:increasing1} for $t = s + 1$. Furthermore, using \cref{eq:increasing2} for $t = s$, which holds by induction, we see that 
	\begin{align}\label{eq:alpha1s+1}
		\alpha_{1, s + 1} = \frac{\langle \tx_0, \tx_s \rangle}{d} = 1 + o_P(1).
	\end{align} 
\vspace{-0.3cm}
\subsubsection*{Proof of Eq.~\eqref{eq:increasing3} for $t = s + 1$}
Next, we prove \cref{eq:increasing3} for $t = s + 1$. 
	We have showed that $y_{s + 1} = \sum_{i = 1}^{s + 1} \alpha_{i, s + 1} z_i + v_{s + 1}$ with $\|v_{s + 1}\|_2 = O_P(d / \sqrt{k})$. For the sake of simplicity, we let $g_{s + 1}  = \sum_{i = 1}^{s + 1} \alpha_{i, s + 1} z_i$. We claim without proof that 
	\begin{align}
		\|f_{s + 1}\|_2^2 \leq & \|g_{s + 1}\|_6^6 + 6\|g_{s + 1}\|_{\infty}^5 \|v_{s + 1}\|_1 + 15 \|g_{s + 1}\|_{\infty}^4 \|v_{s + 1}\|_2^2 + 20 \|g_{s + 1}\|_{\infty}^3 \|v_{s + 1}\|_3^3 \nonumber\\
		& + 15\|g_{s + 1}\|_{\infty}^2 \|v_{s + 1}\|_4^4 + 6 \|g_{s + 1}\|_{\infty} \|v_{s + 1}\|_5^5 + \|v_{s + 1}\|_6^6. \label{eq:fs+1}
	\end{align}
	 Standard application of Gaussian concentration reveals that with high probability, $\|g_{s + 1}\|_{\infty} \leq \log k$. Furthermore, for all $j \in \{2,3,4,5,6\}$ we have $\|v_{s + 1}\|_j \leq \|v_{s + 1}\|_2$ and $\|v_{s + 1}\|_1 \leq \sqrt{k}\|v_{s + 1}\|_2$. Using the law of large numbers, we see that $\|g_{s + 1}\|_6^6 = 15k + o_P(k)$ and $\|g_{s + 1}\|_2^2 = k + o_P(k)$. Plugging the above analysis into \cref{eq:fs+1}, we see that 
	 \begin{align}\label{eq:fs+1-upper}
	 	\|f_{s + 1}\|_2^2 \leq 15k + o_P(k) + O_P(d^6 / k^3) \overset{(i)}{=} 15k + o_P(k),
	 \end{align}
	 where in \emph{(i)} we use the assumption that $d^{3/2} \ll k$. This concludes the proof of \cref{eq:increasing3} for $t = s + 1$. 

\vspace{-0.3cm}
\subsubsection*{Proof of Eq.~\eqref{eq:increasing2} for $t = s + 1$}
	 Finally, we prove \cref{eq:increasing2} for $t = s + 1$. Leveraging \cref{eq:cond-xt} in \cref{lemma:conditioning}, we have
	 \begin{align*}
	 	& \frac{\langle x_0, x_{s + 1} \rangle}{3k} = \frac{\langle h_1, f_{s + 1} \rangle}{3k} = \frac{\langle \tA_1 \tx_0, f_{s + 1} \rangle}{3k} = \frac{1}{3k} \langle z_1, (\sum_{i = 1}^{s + 1} \alpha_{i, s + 1} z_i + v_{s + 1})^3  \rangle \\
	 	= & \frac{\alpha_{1, s + 1}^3}{3k} \langle z_1^4, {1} \rangle + \frac{1}{3k} \Big( \langle z_1, (\sum_{i = 1}^{s + 1} \alpha_{i, s + 1} z_i )^3 \rangle - \alpha_{1, s + 1}^3\langle z_1^4, 1 \rangle \Big) + \frac{1}{k} \langle z_1 g_{s + 1}^2 v_{s + 1}, 1 \rangle \\
	 	&+ \frac{1}{k} \langle z_1 g_{s + 1} v_{s + 1}^2, 1 \rangle + \frac{1}{3k} \langle z_1 v_{s + 1}^3, 1\rangle.
	 \end{align*}
	 Using \cref{eq:alpha1s+1} and the fact that $\sum_{i = 1}^{s + 1} \alpha_{i, s + 1}^2 = 1$, we obtain that $\alpha_{i, s + 1} = o_P(1)$ for all $2 \leq i \leq s + 1$. Therefore, straightforward computation reveals that 
	 \begin{align*}
	 	\frac{1}{3k} \Big( \langle z_1, (\sum_{i = 1}^{s + 1} \alpha_{i, s + 1} z_i )^3 \rangle - \langle z_1^4, 1 \rangle \Big) = o_P(1).
	 \end{align*}
	 Furthermore, 
	 \begin{align}\label{eq:infinity-norm-upper-bound}
	 \begin{split}
	 	 & \frac{1}{k} | \langle z_1 g_{s + 1}^2 v_{s + 1}, 1 \rangle | + | \frac{1}{k} \langle z_1 g_{s + 1} v_{s + 1}^2, 1 \rangle | + | \frac{1}{3k} \langle z_1 v_{s + 1}^3, 1\rangle| \\
	 	 \leq & \frac{1}{\sqrt{k}}\|z_1\|_{\infty} \|g_{s + 1}\|_{\infty}^2 \|v_{s + 1}\|_2 + \frac{1}{k} \|z_1\|_{\infty} \|g_{s + 1}\|_{\infty} \|v_{s + 1}\|_2^2 + \frac{1}{3k}\|z_1\|_{\infty} \|v_{s + 1}\|_2^3 = o_P(1). 
	 \end{split}
	 \end{align}
	 As a result, we conclude that 
	 \begin{align*}
	 	\frac{\langle x_0, x_{s + 1} \rangle}{3k}  = \frac{\alpha_{1, s + 1}^3}{3k} \langle z_1^4, {1} \rangle  + o_P(1) = 1 + o_P(1). 
	 \end{align*}
	 This proves the second part of \cref{eq:increasing2} for $s = t + 1$. 
  
  Again by \cref{eq:cond-xt} in \cref{lemma:conditioning}, we see that 
	 \begin{align*}
	 	\frac{d}{9k^2}\|x_{s + 1}\|_2^2 = \frac{d}{9k^2}\sum_{i = 0}^s \frac{\langle h_{i + 1}, f_{s + 1} \rangle^2}{\|\txp_i\|_2^2} + \frac{d}{9k^2}\|\Pi_{X_{0:s}}^{\perp} \bA_{s + 1}^{\top} \fp_{s + 1}\|_2^2.
	 \end{align*}
	 Recall that we just proved $\langle h_1, f_{s + 1} \rangle^2 / 9k^2 = 1 + o_P(1)$. Furthermore, by the law of large numbers and \cref{eq:increasing3} for $t = s + 1$, we have $d\|\Pi_{X_{0:s}}^{\perp} \bA_{s + 1}^{\top} \fp_{s + 1}\|_2^2 / 9k^2 = O_P(d / k) = o_P(1)$. Therefore, in order to prove the first part of \cref{eq:increasing2} for $s = t + 1$, it suffices to show
	 \begin{align}\label{eq:dhf}
	 	\frac{d\langle h_{i + 1}, f_{s + 1} \rangle^2}{9k^2\|\txp_i\|_2^2} = o_P(1)
	 \end{align}
	 for all $1 \leq i \leq s$. By \cref{eq:ht+1} 
	 \begin{align*}
	 	\frac{\sqrt{d}\langle h_{i + 1}, f_{s + 1} \rangle}{3k\|\txp_i\|_2} =& \frac{\sqrt{d}\langle \fp_i, f_{s + 1} \rangle}{3k\|\txp_i\|_2} \frac{\langle x_i, \txp_i \rangle}{\|\fp_i\|_2^2} + \frac{\sqrt{d}\langle \Pi_{F_{1:i}}^{\perp} \tA_{i + 1} \txp_i, f_{s + 1} \rangle}{3k\|\txp_i\|_2} \\
	 	= &  \frac{\sqrt{d}\langle \fp_i, f_{s + 1} \rangle}{3k\|\fp_i\|_2} \frac{ \|\xp_i\|_2 }{\|\fp_i\|_2} +  \frac{\sqrt{d}\langle \tA_{i + 1} \txp_i, f_{s + 1} \rangle}{3k\|\txp_i\|_2} - \sum_{j = 1}^i \frac{\sqrt{d}\langle \fp_j, \tA_{i + 1} \txp_i \rangle \langle \fp_j, f_{s + 1} \rangle}{3k\|\txp_i\|_2 \|\fp_j\|_2^2}.
	 \end{align*}
	 Using \cref{eq:47.5}, \eqref{eq:fs+1-upper} and Cauchy-Schwarz inequality, we see that
	 \begin{align*}
	 	 \frac{\sqrt{d}\langle \fp_i, f_{s + 1} \rangle}{3k\|\fp_i\|_2} \frac{ \|\xp_i\|_2 }{\|\fp_i\|_2} \leq \frac{\sqrt{d}\| f_{s + 1} \|_2}{3k} \frac{ \|\xp_i\|_2 }{\|\fp_i\|_2} = O_P(\sqrt{d / k}) = o_P(1). 
	 \end{align*}
	 For all $1 \leq j \leq i$, we have 
	 \begin{align*}
	 	\frac{\sqrt{d}\langle \fp_j, \tA_{i + 1} \txp_i \rangle}{\|\txp_i\|_2 \|\fp_j\|_2} \overset{d}{=} \normal(0,1).
	 \end{align*}
	 Therefore, 
	 \begin{align*}
	 	\frac{\sqrt{d}\langle \fp_j, \tA_{i + 1} \txp_i \rangle \langle \fp_j, f_{s + 1} \rangle}{3k\|\txp_i\|_2 \|\fp_j\|_2^2} = O_P(1) \cdot \frac{\langle \fp_j, f_{s + 1} \rangle}{k\|\fp_j\|_2}  = o_P(1). 
	 \end{align*}
	 %
	 %
	 %
	 Note that 
	 \begin{align*}
	 	 & \frac{\sqrt{d}\langle \tA_{i + 1} \txp_i, f_{s + 1} \rangle}{3k\|\txp_i\|_2} \\
    =& \frac{1}{3k} \langle z_{i + 1}, (\sum_{i = 1}^{s + 1} \alpha_{i, s + 1} z_i + v_{s + 1})^3  \rangle \\
	 	= & \frac{\alpha_{i + 1, s + 1}^3}{3k} \langle z_{i + 1}^4, {1} \rangle + \frac{1}{3k} \Big( \langle z_{i + 1}, (\sum_{i = 1}^{s + 1} \alpha_{i, s + 1} z_i )^3 \rangle - \alpha_{i + 1, s + 1}^3\langle z_{i + 1}^4, 1 \rangle \Big) + \frac{1}{k} \langle z_{i + 1} g_{s + 1}^2 v_{s + 1}, 1 \rangle \\
	 	&+ \frac{1}{k} \langle z_{i + 1} g_{s + 1} v_{s + 1}^2, 1 \rangle + \frac{1}{3k} \langle z_{i + 1} v_{s + 1}^3, 1\rangle,
	 \end{align*}
	 which is $o_P(1)$ via a similar argument that is similar to the derivation of \cref{eq:infinity-norm-upper-bound} and $\alpha_{i + 1, s + 1} = o_P(1)$, which we have already proved. Thus, we have completed the proof of \cref{eq:increasing2} for $t = s + 1$.

\section{Gaussian conditioning: Proof of Lemma~\ref{lemma:conditioning}}\label{sec:pf_cond_lem}

   The proof idea comes from \cite[Lemma 3.1]{montanari2022adversarial}. We first show that for all $t \in \NN$, 
   \begin{align}
   		& A = \Pi_{F_{1:t}}^{\perp} \tW_{t + 1} \Pi_{X_{0:t-1}}^{\perp} + \Pi_{F_{1:t}} A \Pi_{X_{0:t-1}}^{\perp} + \Pi_{F_{1:t}}^{\perp} A \Pi_{X_{0:t-1}} + \Pi_{F_{1:t}} A \Pi_{X_{0:t-1}}, \label{eq:cond-W1} \\
   		& A = \Pi_{F_{1:t}}^{\perp} \bar W_{t + 1} \Pi_{X_{0:t}}^{\perp} + \Pi_{F_{1:t}} A \Pi_{X_{0:t}}^{\perp} + \Pi_{F_{1:t}}^{\perp} A \Pi_{X_{0:t}} + \Pi_{F_{1:t}} A \Pi_{X_{0:t}}, \label{eq:cond-W2}
   \end{align}
   where $\tW_{t + 1}\overset{d}{=}\bar{W}_{t + 1} \overset{d}{=} A$, and satisfy $\tW_{t + 1} \perp \cF_{t,t}$, $\bar W_{t + 1} \perp \cF_{t,t+1}$. Next, we prove \cref{eq:cond-W1} and \cref{eq:cond-W2} via induction. For the base case $t = 0$, \cref{eq:cond-W1} holds as we can take $\tW_1 = A$, which is independent of $\cF_{0,0}$ by dedinition. \cref{eq:cond-W2} with $t = 0$ is a direct consequence of Lemma 3.1 in \cite{montanari2022adversarial}. 
   
   Suppose decompositions \eqref{eq:cond-W1} and \eqref{eq:cond-W2} hold for $t = s - 1$, we then prove it also holds for $t = s$ based on induction hypothesis. We first decompose $A$ as the sum of the following four terms:
   \begin{align*}
   		A = \Pi_{F_{1:s}}^{\perp} A \Pi_{X_{0:s-1}}^{\perp} + \Pi_{F_{1:s}} A \Pi_{X_{0:s-1}}^{\perp} + \Pi_{F_{1:s}}^{\perp} A \Pi_{X_{0:s-1}} + \Pi_{F_{1:s}} A \Pi_{X_{0:s-1}}.
   \end{align*}  
   Using induction hypothesis, we see that
   \begin{align*}
   	    & A = \Pi_{F_{1:s - 1}}^{\perp} \bar W_{s} \Pi_{X_{0:s-1}}^{\perp} + \Pi_{F_{1:s - 1}} A \Pi_{X_{0:s - 1}}^{\perp} + \Pi_{F_{1:s - 1}}^{\perp} A \Pi_{X_{0:s - 1}} + \Pi_{F_{1:s - 1}} A \Pi_{X_{0:s - 1}} \\
   	    \implies & \Pi_{F_{1:s - 1}}^{\perp} A \Pi_{X_{0:s-1}}^{\perp} = \Pi_{F_{1:s - 1}}^{\perp} \bar W_{s} \Pi_{X_{0:s-1}}^{\perp} \implies \Pi_{F_{1:s}}^{\perp} A \Pi_{X_{0:s-1}}^{\perp} = \Pi_{F_{1:s}}^{\perp} \bar W_{s} \Pi_{X_{0:s-1}}^{\perp} \\
   		\implies & A = \Pi_{F_{1:s}}^{\perp} \bar{W}_s \Pi_{X_{0:s-1}}^{\perp} + \Pi_{F_{1:s}} A \Pi_{X_{0:s-1}}^{\perp} + \Pi_{F_{1:s}}^{\perp} A \Pi_{X_{0:s-1}} + \Pi_{F_{1:s}} A \Pi_{X_{0:s-1}}.
   \end{align*} 
   Since by induction we have $\bar{W}_s \perp \cF_{s - 1, s}$, we can then conclude that $\bar W_s \perp \{F_{1:s}, Y_{1:s}, X_{0:s-1}\}$. We take 
   \begin{align*}
   	\tW_{s + 1} = \Pi_{F_{1:s}}^{\perp} \bar{W}_s \Pi_{X_{0:s-1}}^{\perp} + \Pi_{F_{1:s}} W' \Pi_{X_{0:s-1}}^{\perp}  + \Pi_{F_{1:s}}^{\perp} W' \Pi_{X_{0:s-1}} + \Pi_{F_{1:s}} W' \Pi_{X_{0:s-1}}, 
   \end{align*}
   where $W'$ is an independent copy of $A$ that is independent of $\cF_{s,s}$. We immediately see that given any specific value of $\{F_{1:s}, X_{0:s-1}, Y_{1:s}, \Pi_{F_{1:s}}\bar{W}_s, \bar{W}_s \Pi_{X_{0:s - 1}}\}$, the conditional distribution of $\tW_{s + 1}$ is equal to the law of $A$. As a result, we deduce that $\tW_{s + 1} \perp \{F_{1:s}, X_{0:s-1},  Y_{1:s}, \Pi_{F_{1:s}}\bar{W}_s, \bar{W}_s \Pi_{X_{0:s - 1}}\}$. Again by induction, we know that 
   \begin{align*}
   		A = \Pi_{F_{1:s - 1}}^{\perp} \bar W_{s} \Pi_{X_{0:s-1}}^{\perp} + \Pi_{F_{1:s - 1}} A \Pi_{X_{0:s - 1}}^{\perp} + \Pi_{F_{1:s - 1}}^{\perp} A \Pi_{X_{0:s - 1}} + \Pi_{F_{1:s - 1}} A \Pi_{X_{0:s - 1}}.
   \end{align*}
   It then follows that
   \begin{align*}
   		A^{\top} f_s =& \Pi_{X_{0:s - 1}}^{\perp} \bar{W}_s^{\top} \Pi_{F_{1:s - 1}}^{\perp} f_s + \Pi_{X_{0:s - 1}}^{\perp}  A^{\top} \Pi_{F_{1:s - 1}} f_s + \Pi_{X_{0:s - 1}} A^{\top} \Pi_{F_{1:s - 1}}^{\perp} f_s + \Pi_{X_{0:s - 1}} A^{\top} \Pi_{F_{1:s - 1}} f_s \\
   		= & \bar{W}_s^{\top} f_s - \bar{W}_s^{\top} F_{1:s-1}(F_{1:s - 1}^{\top} F_{1:s - 1})^{\dagger} F_{1:s - 1}^{\top} f_s - \Pi_{X_{0:s - 1}} \bar{W}_s^{\top} (I - F_{1:s-1}(F_{1:s - 1}^{\top} F_{1:s - 1})^{\dagger} F_{1:s - 1}^{\top}) f_s + \\
   		& X_{0:s - 1}(X_{0:s - 1}^{\top} X_{0:s - 1})^{\dagger} D Y_{1:s}^{\top},
   \end{align*}
   where $D = \diag(\{ \|x_i\|_2^2 / d \}_{0 \leq i \leq s - 1})$. Therefore, we see that 
   $$x_s = A^{\top} f_s \in \sigma ( \{ F_{1:s}, X_{0:s-1},  Y_{1:s}, {F_{1:s}}\bar{W}_s, \bar{W}_s {X_{0:s - 1}}\} ).$$  
   This further implies that $\tW_{s + 1} \perp \sigma(\{ A X_{0:s - 1}, A^{\top} F_{1:s}, X_{0:s - 1}, Y_{1:s}, F_{1:s} \}) = \sigma(\{X_{0:s}, Y_{1:s}\}) = \cF_{s, s}$, which concludes the proof of \cref{eq:cond-W1} for $t = s$. The proof for \cref{eq:cond-W2} for $t = s$ can be shown similarly. 
   
   Next, we prove \cref{lemma:conditioning} using \cref{eq:cond-W1} and \cref{eq:cond-W2}. For $t \in \NN$, we define 
   \begin{align}
   		& \tA_{t + 1} = \Pi_{F_{1:t}}^{\perp} \tilde{W}_{t + 1} \Pi_{\tx_{t}^{\perp}} + \Pi_{F_{1:t}} \tilde{M}_{t + 1} \Pi_{\tx_{t}^{\perp}} + \Pi_{F_{1:t}} \tilde{M}_{t + 1} \Pi_{\tx_{t}^{\perp}}^{\perp} + \Pi_{F_{1:t}}^{\perp} \tilde{M}_{t + 1} \Pi_{\tx_{t}^{\perp}}^{\perp}, \label{eq:8} \\
   		& \bA_{t + 1} = \Pi_{X_{0:t}}^{\perp} \bar W_{t + 1}^{\top} \Pi_{\fp_{t + 1}} + \Pi_{X_{0:t}}^{\perp} \bar M_{t + 1}^{\top} \Pi_{\fp_{t + 1}}^{\perp} + \Pi_{X_{0:t}} \bar M_{t + 1}^{\top} \Pi_{\fp_{t + 1}}^{\perp} + \Pi_{X_{0:t}} \bar M_{t + 1}^{\top} \Pi_{\fp_{t + 1}}. \label{eq:9}
   \end{align}
   where $\tilde{M}_{t + 1} \overset{d}{=} \bar{M}_{t + 1} \overset{d}{=} A$ are i.i.d., and independent of $\sigma(\cF_{T, T} \cup \sigma ((\bar{W}_t, \tW_t)_{0 \le t \le T + 1}))$. From \cref{eq:cond-W1,eq:cond-W2} we see that 
   \begin{align}
   		& \Pi_{F_{1:t}}^{\perp} A \Pi_{\tx_{t}^{\perp}} = \Pi_{F_{1:t}}^{\perp} \tilde{W}_{t + 1} \Pi_{\tx_{t}^{\perp}} = \Pi_{F_{1:t}}^{\perp} \tilde{A}_{t + 1} \Pi_{\tx_{t}^{\perp}} \in \cF_{t, t + 1}, \label{eq:10} \\
   		& \Pi_{X_{0:t}}^{\perp} A^{\top} \Pi_{\fp_{t + 1}} = \Pi_{X_{0:t}}^{\perp} \bar W_{t + 1}^{\top} \Pi_{\fp_{t + 1}} = \Pi_{X_{0:t}}^{\perp} \bar A_{t + 1}^{\top} \Pi_{\fp_{t + 1}} \in \cF_{t + 1, t + 1}. \label{eq:11}
   \end{align}
   We then show that $\{\tA_t, \bA_t: t \in [T]\}$ are i.i.d. with marginal distribution $A$. To this end, we only need to show: (1) For all $t \in [T]$, $\tA_t$ is independent of $\tA_1, \cdots, \tA_{t - 1}, \bA_1, \cdots, \bA_{t - 1}$ and has marginal distribution $A$; (2) For all $t \in [T]$, $\bA_t$ is independent of $\tA_1, \cdots, \tA_{t - 1}, \bA_1, \cdots, \bA_{t - 1}$ and has marginal distribution $A$. 
   
   We prove this result via induction. For the base case $t = 1$, $\tA_1 = \tW_1 \Pi_{x_0} + \tilde{M}_1 \Pi_{x_0}^{\perp}$ which obviously has marginal distribution equal to $A$ as $\tW_1 \perp \cF_{0,0}$. On the other hand, $\bA_1 = \Pi_{x_0}^{\perp} \bar W_{1}^{\top} \Pi_{f_1} + \Pi_{x_0}^{\perp} \bar M_1^{\top} \Pi_{f_1}^{\perp} + \Pi_{x_0} \bar M_1^{\top} \Pi_{f_1} + \Pi_{x_0} \bar{M}_1^{\top} \Pi_{f_1}^{\perp}$ and $\bar{W}_1$ is independent of $\cF_{0,1}$, thus the conditional distribution of $\bA_1$ conditioning on $\cF_{0,1}$ is always equal to $A$. The marginal distribution of $\bar{A}_1$ being $A$ follows as a simple corollary. Notice that $\tA_1 \in \sigma(\cF_{0,1} \cup \sigma(\tilde M_1))$ and $\cF_{0,1} \perp \tilde{M}_1$. Therefore, in order to show $\bA_1 \perp \tA_1$, it suffices to prove $\bA_1 \perp \tilde{M}_1$ and $\bA_1 \perp \cF_{0,1}$. The first independence follows by definition. As for the second independence, since $x_0, f_1 \in \cF_{0,1}$ and $\bar W_1, \bar{M}_1 \perp \cF_{0,1}$, we can conclude that conditioning on $\cF_{0,1}$, the conditional distribution of $\bA_1$ is always equal to its marginal distribution, thus concludes the proof of this step. 
   
   Suppose the result holds for the first $t$ steps, we then prove it also holds for $t + 1$ via induction. Conditioning on $\cF_{t,t}$, we see that the conditional distribution of $\tA_{t + 1}$ is always equal to the marginal distribution of $A$, thus $\tA_{t + 1}$ has marginal distribution equal to $A$ and $\tA_{t + 1} \perp \cF_{t,t}$. Notice that $\tA_i \in \sigma(\cF_{i - 1, i} \cup \sigma(\tilde{M}_i))$ and $\bA_i \in \sigma(\cF_{i,i} \cup \sigma(\bar{M}_i))$. Therefore, in order to prove $\tA_{t + 1} \perp \{\tA_1, \cdots, \tA_t, \bA_1, \cdots, \bA_t\}$, it suffices to prove $\tA_{t + 1} \perp \cF_{t,t}$ and $\tA_{t + 1} \perp\sigma(\{\tilde{M}_i, \bar{M}_i: i \in [t]\})$. We have already proved the first independence. The second independence follows by definition.
   
   As for $\bA_{t + 1}$, first notice that the conditional distribution of $\bA_{t + 1}$ conditioning on $\cF_{t, t + 1}$ is always equal to the marginal distribution of $A$, thus $\bA_{t + 1}$ has marginal distribution equal to $A$ and $\bA_{t + 1} \perp \cF_{t, t + 1}$.  Similarly, notice that $\tA_i \in \sigma(\cF_{i - 1, i} \cup \sigma(\tilde{M}_i))$ and $\bA_i \in \sigma(\cF_{i,i} \cup \sigma(\bar{M}_i))$. As a result, in order to prove $\bA_{t + 1} \perp \{\tA_1, \cdots, \tA_{t + 1}, \bA_1, \cdots, \bA_{t}\}$, we only need to show $\bA_{t + 1} \perp \cF_{t,t + 1}$ and $\bA_{t + 1} \perp\sigma(\{\tilde{M}_1, \cdots, \tilde{M}_{t + 1}, \bar{M}_1, \cdots, \bar{M}_t\})$. We have already proved the first independence. The second independence follows by definition. Thus, we have concluded the proof of arguments (1) and (2) via induction. 
   
   Finally, we are ready to prove the lemma. We first show that 
   \begin{align}\label{eq:prop-ht}
   		h_{t + 1} = A\txp_t
   \end{align}
   for all $0 \leq t \leq T - 1$. By \cref{eq:cond-W1} and \cref{eq:10}, we have
   \begin{align*}
   		A \txp_t = &  \Pi_{F_{1:t}}^{\perp} \tW_{t + 1} \Pi_{X_{0:t - 1}}^{\perp} \txp_t + \Pi_{F_{1:t}} A \txp_t \\
   		= & \Pi_{F_{1:t}}^{\perp} \tA_{t + 1} \txp_t + \sum_{i = 1}^t \fp_i \cdot \frac{\langle A^{\top} \fp_i, \txp_t \rangle}{\|\fp_i\|_2^2} \\
   		= & \Pi_{F_{1:t}}^{\perp} \tA_{t + 1} \txp_t + \fp_t \cdot \frac{\langle x_t, \txp_t \rangle}{\| \fp_t \|_2^2} = h_{t + 1},
   \end{align*}
   which concludes the proof of \cref{eq:prop-ht}. Now, by definition, we deduce that
   \begin{align*}
       x_t =\, & A^\top f_t = \Pi_{X_{0: t-1}}^{\perp} A^\top f_t + \sum_{i=0}^{t-1} \frac{\langle \tilde{x}_i^{\perp}, A^\top f_t \rangle}{\norm{\tilde{x}_i^{\perp}}_2^2} \cdot \tilde{x}_i^{\perp} \\
       \stackrel{(i)}{=}\, & \Pi_{X_{0: t-1}}^{\perp} A^\top f_t^{\perp} + \sum_{i=0}^{t-1} \frac{\langle A \tilde{x}_i^{\perp}, f_t \rangle}{\norm{\tilde{x}_i^{\perp}}_2^2} \cdot \tilde{x}_i^{\perp} \\
       =\, & \Pi_{X_{0: t-1}}^{\perp} A^\top f_t^{\perp} + \sum_{i=0}^{t-1} \frac{\langle h_{i+1}, f_t \rangle}{\norm{\tilde{x}_i^{\perp}}_2^2} \cdot \tilde{x}_i^{\perp} \\
       \stackrel{(ii)}{=}\, & \Pi_{X_{0: t-1}}^{\perp} {\bA}_t^\top f_t^{\perp} + \sum_{i=0}^{t-1} \frac{\langle h_{i+1}, f_t \rangle}{\norm{\tilde{x}_i^{\perp}}_2^2} \cdot \tilde{x}_i^{\perp},
   \end{align*}
   where $(i)$ follows from the fact that $\Pi_{X_{0: t-1}}^{\perp} A^\top f_i = \Pi_{X_{0: t-1}}^{\perp} x_i = 0$ for $0 \le i \le t - 1$, and $(ii)$ follows from Eq.~\eqref{eq:11}. Similarly, we obtain that
   \begin{align*}
       y_{t + 1} =\, & A \tilde{x}_t = \sum_{i = 1}^{t + 1} \frac{\langle \tilde{x}_{i - 1}^{\perp}, \tilde{x}_t \rangle}{\norm{\tilde{x}_{i - 1}^{\perp}}_2^2} \cdot A \tilde{x}_{i - 1}^{\perp} = \sum_{i = 1}^{t + 1} \frac{\langle \tilde{x}_{i - 1}^{\perp}, \tilde{x}_t \rangle}{\norm{\tilde{x}_{i - 1}^{\perp}}_2^2} \cdot h_i \\
       =\, & \sum_{i = 1}^{t + 1} \Pi_{F_{1:i - 1}}^{\perp} \tilde{A}_i \tx_{i - 1}^{\perp} \frac{\langle \txp_{i - 1}, \tx_t \rangle}{\|\tx_{i - 1}^{\perp}\|_2^2} + \sum_{i = 2}^{t + 1} \fp_{i - 1} \cdot \frac{\langle x_{i - 1}, \txp_{i - 1} \rangle}{\|\fp_{i - 1}\|_2^2} \cdot \frac{\langle \txp_{i - 1}, \tx_t\rangle}{\|\txp_{i - 1}\|_2^2}.
   \end{align*}
   This completes the proof of \cref{lemma:conditioning}.

\section{Supporting lemmas}

This section contains supporting lemmas required by our proof. 

\begin{lem}[Tails of the normal distribution, Proposition 2.1.2 of \cite{vershynin2018high}]\label{lemma:gaussian-tail}
	Let $g \sim \normal(0,1)$. Then for all $t > 0$ we have
	\begin{align*}
		\left(\frac{1}{t} - \frac{1}{t^3}\right) \cdot \frac{1}{\sqrt{2\pi}} e^{-t^2 / 2} \leq \P(g \geq t) \leq \frac{1}{t} \cdot \frac{1}{\sqrt{2\pi}} e^{-t^2 / 2}.
	\end{align*}
\end{lem}

\begin{lem}[Bernstein's inequality, Theorem 2.8.1 of \cite{vershynin2018high}]\label{lemma:bernstein}
	Let $X_1, \cdots, X_N$ be independent, mean zero, sub-exponential random variables. Then, for every $t \geq 0$, we have
	\begin{align*}
		\P\left( \Big| \sum_{i = 1}^N X_i \Big| \geq t \right) \leq 2 \exp \left[ -c \min\left( \frac{t^2}{\sum_{i = 1}^N \|X_i\|_{\Psi_1}^2}, \frac{t}{\max_i \|X_i\|_{\Psi_1}} \right) \right],
	\end{align*}
	where $c > 0$ is an absolute constant, and $\|\cdot\|_{\Psi_1}$ is the Orlicz norm. 
\end{lem}

\begin{lem}[Concentration inequality for sub-Weibull distribution, adapted from Theorem 3.1 of \cite{hao2019bootstrapping}]\label{lemma:conc_non_exp} Let $\{ X_i \}_{1 \le i \le n}$ be a sequence of i.i.d. random variables such that for some constants $C_1, C_2 > 0$ and $q \in (0, 1)$, $\P ( \vert X_1 \vert \ge t ) \le C_1 \exp(- C_2 t^q)$ for all sufficiently large $t > 0$. Then, there exists a constant $C > 0$, such that the following holds for all $t > 0$:
\begin{equation*}
	\P \left( \left\vert \sum_{i=1}^{n} \left( X_i - \E [X_i] \right) \right\vert \ge t \right) \le C \exp \left(-\min \left\{ \frac{C^{-2} t^2}{n}, \ C^{-q} t^q \right\}\right).
\end{equation*}
\end{lem}

\begin{proof}
    Choosing $a$ to be the vector of all ones in \cite[Theorem 3.1]{hao2019bootstrapping}, we know that there exists a constant $C > 0$, such that
    \begin{equation*}
        \P \left( \left\vert \sum_{i=1}^{n} \left( X_i - \E [X_i] \right) \right\vert \ge C \left( \sqrt{n \log (1 / \alpha)} + (\log(1 / \alpha))^{1/q} \right) \right) \le \alpha.
    \end{equation*}
    Set $t = C \left( \sqrt{n \log (1 / \alpha)} + (\log(1 / \alpha))^{1/q} \right)$, then one of the followings must happen:
    \begin{itemize}
        \item [(a)] $C \sqrt{n \log (1 / \alpha)} \ge t/2$. This implies that
        \begin{equation*}
            \alpha \le \exp \left( - \frac{t^2}{4 C^2 n} \right).
        \end{equation*}
        
        \item [(b)] $C (\log(1 / \alpha))^{1/q} \ge t/2$, which is equivalent to
        \begin{equation*}
            \alpha \le \exp \left( - \frac{t^q}{(2C)^q} \right).
        \end{equation*}
    \end{itemize}
    Hence, we finally obtain that
    \begin{align*}
        \P \left( \left\vert \sum_{i=1}^{n} \left( X_i - \E [X_i] \right) \right\vert \ge t \right) & \le \alpha \le \max \left\{ \exp \left( - \frac{t^2}{4 C^2 n} \right), \ \exp \left( - \frac{t^q}{(2C)^q} \right) \right\} \\
        & = \exp \left( - \min \left\{ \frac{t^2}{4 C^2 n}, \ \frac{t^q}{(2C)^q} \right\} \right).
    \end{align*}
    This completes the proof.
\end{proof}

\end{document}